\RequirePackage[utf8]{inputenc}
\makeatletter
\def\cup@reference@code{%
  \RequirePackage[
    style=numeric,
    sorting=none,
    backend=biber,
    natbib,
    maxbibnames=99
  ]{biblatex}%
  \renewcommand*{\bibfont}{\footnotesize}%
}
\makeatother

\documentclass[
  journal=medium,
  manuscript=article,
  manuscriptlabel={Research Preprint},
  logo=false,
  year=2026,
  volume=1
]{cup-journal}

\usepackage[T1]{fontenc}
\usepackage{microtype}
\usepackage{amsmath,amssymb,amsthm}
\usepackage{mathtools}
\usepackage{booktabs}
\usepackage{adjustbox}
\usepackage{makecell,multirow}
\usepackage{subcaption}
\usepackage{orcidlink}
\usepackage{needspace}
\definecolor{primary}{HTML}{006BA2}
\definecolor{accent}{HTML}{E3120B}
\definecolor{neutral}{HTML}{4A4A4A}
\definecolor{boxfill}{HTML}{E8F4FD}
\definecolor{boxborder}{HTML}{006BA2}
\definecolor{defcolor}{HTML}{2E7D32}
\definecolor{todobg}{HTML}{FFF4E5}
\definecolor{todoframe}{HTML}{B45309}

\usepackage{thmtools}
\usepackage{tcolorbox}
\tcbuselibrary{skins,breakable}
\usepackage{hyperref}
\usepackage{cleveref}

\RawFloats[figure,table]

\ifpdf
  \DeclareGraphicsExtensions{.pdf,.png,.jpg}
\else
  \DeclareGraphicsExtensions{.eps}
\fi

\newcommand{\R}{\mathbb{R}}
\newcommand{\E}{\mathbb{E}}
\DeclarePairedDelimiter{\norm}{\lVert}{\rVert}

\DeclareMathOperator{\rank}{rank}

\hypersetup{
  colorlinks=true,
  linkcolor=primary,
  citecolor=primary,
  urlcolor=accent,
  bookmarksnumbered=true,
  bookmarksopen=true,
  pdflang={en-US}
}

\crefname{equation}{equation}{equations}
\Crefname{equation}{Equation}{Equations}
\crefname{figure}{Figure}{Figures}
\Crefname{figure}{Figure}{Figures}
\crefname{table}{Table}{Tables}
\Crefname{table}{Table}{Tables}
\crefname{section}{section}{sections}
\Crefname{section}{Section}{Sections}
\crefname{theorem}{theorem}{theorems}
\Crefname{theorem}{Theorem}{Theorems}
\crefname{lemma}{lemma}{lemmas}
\Crefname{lemma}{Lemma}{Lemmas}
\crefname{proposition}{proposition}{propositions}
\Crefname{proposition}{Proposition}{Propositions}
\crefname{corollary}{corollary}{corollaries}
\Crefname{corollary}{Corollary}{Corollaries}
\crefname{conjecture}{conjecture}{conjectures}
\Crefname{conjecture}{Conjecture}{Conjectures}
\crefname{remark}{remark}{remarks}
\Crefname{remark}{Remark}{Remarks}
\crefname{openproblem}{open problem}{open problems}
\Crefname{openproblem}{Open Problem}{Open Problems}

\makeatletter
\@ifpackageloaded{xpatch}{}{\usepackage{xpatch}}
\def\cup@journal@name{Product-Aware Deterministic Rounding}
\def\cup@manuscript{preprint}
\patchcmd{\@maketitle}{\vspace*{\baselineskip}}{\vspace*{0.2\baselineskip}}{}{}
\patchcmd{\@maketitle}{(\cup@year), {\volumefont\cup@vol}, \thepage--\pageref{LastPage}}{}{}{}
\renewcommand*\cup@maketitle@extras@hook{%
  \begingroup
  \renewcommand{\thefootnote}{\fnsymbol{footnote}}%
  \footnotetext{This manuscript has been authored by UT-Battelle, LLC under
  Contract No.\ DE-AC05-00OR22725 with the U.S. Department of Energy. The
  publisher, by accepting the article for publication, acknowledges that the
  United States Government retains a non-exclusive, paid-up, irrevocable,
  world-wide license to publish or reproduce the published form of this
  manuscript, or allow others to do so, for United States Government purposes.
  The Department of Energy will provide public access to these results of
  federally sponsored research in accordance with the DOE Public Access Plan
  (\url{https://www.energy.gov/doe-public-access-plan}).}%
  \endgroup
}

\renewcommand\appendix{\par
  \setcounter{section}{0}%
  \setcounter{subsection}{0}%
  \gdef\thesection{\@Alph\c@section}%
  \gdef\theHsection{\@Alph\c@section}%
  \gdef\theHsubsection{\theHsection.\arabic{subsection}}%
  \gdef\section@cntformat{Appendix \thesection.\quad}%
}
\makeatother

\declaretheoremstyle[
  headfont=\bfseries\color{primary},
  notefont=\normalfont,
  bodyfont=\itshape,
  headpunct={.},
  spaceabove=0pt,
  spacebelow=0pt
]{cupthmplain}
\declaretheoremstyle[
  headfont=\bfseries\color{defcolor},
  notefont=\normalfont,
  bodyfont=\normalfont,
  headpunct={.},
  spaceabove=0pt,
  spacebelow=0pt
]{cupthmdef}
\declaretheoremstyle[
  headfont=\itshape\bfseries\color{neutral},
  notefont=\normalfont,
  bodyfont=\normalfont,
  headpunct={.},
  spaceabove=0pt,
  spacebelow=0pt
]{cupthmrem}

\declaretheorem[style=cupthmplain,name=Theorem,numberwithin=section]{theorem}
\declaretheorem[style=cupthmplain,name=Lemma,sibling=theorem]{lemma}

\declaretheorem[style=cupthmplain,name=Corollary,sibling=theorem]{corollary}

\declaretheorem[style=cupthmdef,name=Open Problem,sibling=theorem]{openproblem}
\declaretheorem[style=cupthmrem,name=Remark,sibling=theorem]{remark}

\tcbset{
  cupresultbox/.style={
    enhanced,breakable,sharp corners,boxrule=0pt,frame hidden,
    colback=primary!5,borderline west={2pt}{0pt}{primary},
    left=7pt,right=6pt,top=4pt,bottom=4pt,
    before skip=6pt,after skip=6pt
  },
  cupdefinitionbox/.style={
    enhanced,breakable,sharp corners,boxrule=0pt,frame hidden,
    colback=defcolor!6,borderline west={2pt}{0pt}{defcolor},
    left=7pt,right=6pt,top=4pt,bottom=4pt,
    before skip=6pt,after skip=6pt
  },
  cupremarkbox/.style={
    enhanced,breakable,sharp corners,boxrule=0pt,frame hidden,
    colback=neutral!4,borderline west={2pt}{0pt}{neutral},
    left=7pt,right=6pt,top=4pt,bottom=4pt,
    before skip=6pt,after skip=6pt
  }
}
\tcolorboxenvironment{theorem}{cupresultbox}
\tcolorboxenvironment{lemma}{cupresultbox}
\tcolorboxenvironment{proposition}{cupresultbox}
\tcolorboxenvironment{corollary}{cupresultbox}
\tcolorboxenvironment{conjecture}{cupresultbox}
\tcolorboxenvironment{definition}{cupdefinitionbox}
\tcolorboxenvironment{assumption}{cupdefinitionbox}
\tcolorboxenvironment{openproblem}{cupdefinitionbox}
\tcolorboxenvironment{remark}{cupremarkbox}

\title[Product-Aware Deterministic Rounding]{Product-Aware Deterministic
Rounding for Quantized Matrix Multiplication}
\author{Piyush Sao\,\orcidlink{0000-0002-9432-5855}}
\email[Piyush Sao]{saopk@ornl.gov}
\author{Narasinga Miniskar\,\orcidlink{0000-0001-8259-8891}}
\author{Pedro Valero-Lara\,\orcidlink{0000-0002-1479-4310}}
\author{Keita Teranishi\,\orcidlink{0000-0001-6647-2690}}
\author{Sudip Seal\,\orcidlink{0000-0003-3233-0656}}
\affiliation{Oak Ridge National Laboratory, Oak Ridge, Tennessee 37831, USA}
\keywords{matrix quantization; discrepancy theory; deterministic rounding;
matrix products; post-training quantization}
\hypersetup{
  pdftitle={Product-Aware Deterministic Rounding for Quantized Matrix Multiplication},
  pdfauthor={Piyush Sao, Narasinga Miniskar, Pedro Valero-Lara, Keita Teranishi, Sudip Seal},
  pdfsubject={Deterministic product-aware rounding in dynamic and static matrix settings},
  pdfkeywords={matrix quantization, discrepancy theory, deterministic rounding, matrix products, post-training quantization}
}

\begin{document}
\maketitle

\begin{abstract}
Quantized matrix multiplication usually rounds each scalar independently,
although matrix multiplication combines the resulting errors.  We study
deterministic rounding after scales, clipping bounds, and grids are fixed,
so each active scalar has two admissible choices: round down or round up.

For dynamic activation rounding, we prove that a rank-$r$ weight block lets
all but at most $r$ fractional decisions be resolved while preserving the
relaxed product.  Conditional-expectation completion yields a deterministic
polynomial-time construction with squared product error at most
$\operatorname{OPT}_{\mathrm{dyn}}+r\nu_{\max}^2/4$, where
$\operatorname{OPT}_{\mathrm{dyn}}$ is the best admissible error and
$\nu_{\max}$ is the largest active row norm after weighting by the local
grid gap.  This block-specific guarantee is most favorable when row energy
is distributed across many rows.  For static weight rounding, the exact
expected product-loss metric is the uncentered input second moment when
output bias is fixed; free bias recalibration reduces the metric to centered
covariance.  Exact optimization is NP-hard even at rank one.

In balanced synthetic blocks, coordinated error decreases as row count grows
relative to rank.  At $K=1024,r=16$, conditional-expectation completion reaches
a dither-normalized median error of $0.010$, compared with $0.899$ for
round-to-nearest.  Clipping-aware initialization reduces median normalized error by a
factor of $43.4$ at ten-percent clipping.  On Digits, the bias convention
determines the outcome: centered fixed-bias rounding has higher median
held-out product error than round-to-nearest in all four tested bit-width
and calibration-size settings, while uncentered
and centered-plus-bias rounding have lower medians.

\end{abstract}

\section{Introduction}
\label{sec:disc-introduction}

Quantized matrix multiplication usually rounds each scalar independently.
These choices can be suboptimal for the product: matrix
multiplication combines scalar errors, so errors from different coordinates
can cancel.  Consider the product of the row $x=(0.45,\,0.45)$ with the
two-row block $(1,\,2)^\top$.  Rounding both entries to nearest gives the
quantized product $0$ and error $-1.35$; rounding the first entry up and the
second down gives product $1$ and error $-0.35$.  Under an entrywise rule,
errors may still cancel by chance, but the choice at
one entry does not depend
on the weighted error created by another.

To coordinate these choices, we fix all scales, clipping bounds, and grid
levels before rounding $Y=XW$.  Every scalar strictly between adjacent
levels then has two admissible choices.  In \emph{dynamic rounding}, an
activation row $x$ is available at run time, and its rounded copy may depend
on the weight block $W_{\mathrm{blk}}$ it multiplies.  The objective is
$\norm*{(\widehat x-x)W_{\mathrm{blk}}}_2^2$.  This regime measures the
cancellation available when both operands of a row--block product are known,
providing a per-block guarantee that reusable rounding rules can target.  In
\emph{static rounding}, one rounded weight matrix $\widehat W$ is stored and
reused across future inputs, with objective
$\E\norm*{x(\widehat W-W)}_2^2$.  \Cref{fig:disc-regimes} summarizes the
information and reuse constraints.

A rounding rule should respect these fixed choices, account for their joint
product error, and provide a guarantee at feasible computational cost.
Round-to-nearest satisfies the scalar constraint but chooses each entry
without reference to the other entries' product errors.  Enumerating all
binary patterns finds the best product but uses exponentially many choices.
Calibration-based weight reconstruction already coordinates decisions using
activation data, as in AdaRound and GPTQ \cite{adaround,gptq}.  The questions
here are how low rank gives a deterministic additive certificate for an
observed row--block product, and how the bias convention determines the
exact expected loss for reusable weights.

Our dynamic construction first minimizes a continuous relaxation, then moves
along directions that preserve the relaxed product until at most $r$
decisions remain fractional.  It completes each remaining choice by keeping
the endpoint with the smaller conditional expected error.  This gives a
deterministic additive guarantee of $r\nu_{\max}^2/4$, where $r$ is the
active gap-weighted block's rank and $\nu_{\max}$ its largest row norm.
The effective row count in \Cref{sec:disc-dynamic-comparison} identifies when
this guarantee is small: the product energy must be distributed across many
rows relative to rank.

For static rounding, expected product error separates into a systematic shift
from the input mean and variation around that shift.  Fixed output bias
retains both terms, giving the uncentered second moment as the exact metric.
Free bias recalibration absorbs the mean term, leaving centered covariance.
The distinction predicts a failure mode observed on Digits: centering with
fixed bias yields higher median held-out product error than round-to-nearest
in every tested setting.  Exact optimization is NP-hard even at rank one,
which motivates the additive form of the dynamic guarantee.

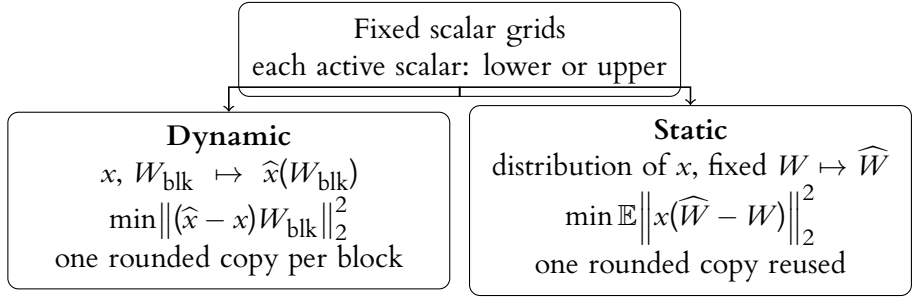
\begin{figure}[!tbp]
  \centering
  \begin{tikzpicture}[
    box/.style={draw, rounded corners, align=center, inner sep=5pt,
      text width=0.30\linewidth, minimum height=0.95cm},
    flow/.style={->, line width=0.6pt}
  ]
    \node[box] (fixed) at (0,0)
      {Fixed scalar grids\\each active scalar: lower or upper};
    \node[box] (dynamic) at (-3.05,-2.0)
      {\textbf{Dynamic}\\
       $x$, $W_{\mathrm{blk}}\mapsto\widehat x(W_{\mathrm{blk}})$\\[1pt]
       $\min\norm*{(\widehat x-x)W_{\mathrm{blk}}}_2^2$\\[1pt]
       one rounded copy per block};
    \node[box] (static) at (3.05,-2.0)
      {\textbf{Static}\\
       distribution of $x$, fixed $W\mapsto\widehat W$\\[1pt]
       $\min\E\norm*{x(\widehat W-W)}_2^2$\\[1pt]
       one rounded copy reused};
    \draw[flow] (fixed.south) -- (0,-0.5) -| (dynamic.north);
    \draw[flow] (fixed.south) -- (0,-0.5) -| (static.north);
  \end{tikzpicture}
  \caption{The two rounding regimes.  Both coordinate the same lower-or-upper
  scalar choices, but dynamic rounding re-rounds an observed activation row
  for each weight block it multiplies, whereas static rounding commits to one
  rounded weight matrix for a distribution of future inputs.}
  \label{fig:disc-regimes}
\end{figure}

\paragraph{Contributions.}

\begin{enumerate}
\item \textbf{Dynamic rounding.}  We express the choice of up/down
  decisions as binary least squares (\Cref{sec:disc-binary-objective}) and give a
  polynomial-time deterministic algorithm whose squared error is at most
  the binary optimum plus $r\nu_{\max}^2/4$
  (\Cref{thm:disc-additive}).  Here $r$ is the rank of the active
  gap-weighted block, and $\nu_{\max}$ is its largest row norm.
\item \textbf{Static rounding.}  For weights rounded once and reused across
  inputs, expected product error is measured by the uncentered second moment
  $M=\E[x^\top x]$, with minimization separating by column whenever the
  fixed feasible set does; free output-bias recalibration replaces $M$ by the
  centered covariance (\Cref{thm:disc-static,rem:disc-bias}).
\item \textbf{Complexity.}  Exact optimization remains hard even at
  rank one.  Reductions from \textsc{Partition} and \textsc{Max-Cut}
  establish barriers to exact and multiplicative guarantees, explaining
  the additive form above (\Cref{thm:disc-partition,thm:disc-maxcut}).
\item \textbf{Empirical validation.}  In balanced synthetic blocks, coordinated
  error decreases as the row count grows relative to rank.  Concentrated
  row energy weakens the certified bound, and clipping-aware initialization
  reduces the measured normalized error by a factor of $43.4$ at ten-percent clipping.
  On Digits, centering without bias correction gives higher median
  held-out error than round-to-nearest in every tested configuration
  (\Cref{sec:disc-experiments}).
\end{enumerate}

The constructions make product-aware cancellation explicit under each reuse
constraint.  Retaining this cancellation in one rounded activation copy
shared across blocks is the central next problem
(\Cref{op:disc-reuse}).

\paragraph{Organization.}
\Cref{sec:disc-setting} fixes the scalar feasible sets and the
independent-decision references.  \Cref{sec:disc-dynamic} treats dynamic
rounding, \Cref{sec:disc-static} the static metric, and
\Cref{sec:disc-hardness} the complexity barriers.
\Cref{sec:disc-experiments} reports the experiments.
\Cref{sec:disc-related} connects the results to earlier quantizers,
and discrepancy theory; \Cref{sec:disc-discussion}
discusses limitations and open problems.  \Cref{sec:disc-extensions}
collects extensions toward reusable rounding.

\section{Fixed-grid rounding choices}
\label{sec:disc-setting}

With scales, clipping bounds, and grids fixed, each coordinate has either
two adjacent-level choices or one forced value.  We first collect these
choices into an error vector, then define the independent-rounding scales
used to measure the benefit of coordination.

\subsection{Adjacent-level feasible sets}

Let a scalar $z$ be quantized on a fixed grid of representable values
and within a fixed clipping interval.  If $z$ lies
strictly between adjacent levels $\ell<u$, write
\[
  z=\ell+\theta d,\qquad
  d=u-\ell>0,\qquad \theta\in(0,1).
\]
The fraction $\theta$ measures how far $z$ lies from the lower level toward
the upper one; we call it the \emph{offset fraction}.  We call such a coordinate
\emph{active}: choosing the lower or upper level is
equivalent to choosing $\chi\in\{0,1\}$, and the resulting error is
\begin{equation}
  \widehat z-z=d(\chi-\theta).
  \label{eq:disc-scalar-error}
\end{equation}
For example, on the integer grid ($d=1$), the value $z=3.7$ has
$\ell=3$, $u=4$, and $\theta=0.7$.  Rounding down leaves error $-0.7$,
whereas rounding up leaves error $+0.3$; both are
$d(\chi-\theta)$ with $\chi\in\{0,1\}$.  We study how to coordinate many such
errors against the rows they multiply.

For a vector of $K$ scalars, the non-active coordinates are fixed.  An
exactly representable value has zero error; clipping forces the error at
the relevant endpoint.  These fixed errors must enter the product along
with the errors we can still choose.

Let $U=\{k_1<\cdots<k_{|U|}\}$ index the active choices in increasing
coordinate order, and let $F$ index the fixed coordinates.  All vectors
indexed by $U$ use this order.
Collect the fixed errors in $e^{\mathrm{fix}}\in\R^K$, with zeros on $U$.
The matrix $D\in\R^{K\times |U|}$ puts each adjustable error back in its
original coordinate and multiplies it by the local step size: column $j$
is $d_{k_j}$ times the $k_j$th standard basis vector.
Collecting these two-way choices into one vector gives
\begin{equation}
  e=e^{\mathrm{fix}}+D(\chi-\theta),
  \qquad \chi\in\{0,1\}^{|U|},\quad
  \theta\in(0,1)^{|U|}.
  \label{eq:disc-affine-error}
\end{equation}
This affine form covers unequal step sizes.  Uniform signed-integer grids
are a special case with $d_k=\Delta$ on all active coordinates.

We use \eqref{eq:disc-affine-error} in two ways.  In the dynamic regime,
$e$ is the rounding error in one observed activation row and is coordinated
for a fixed weight block $W_{\mathrm{blk}}\in\R^{K\times p}$, where $p$ is
the output block width.  We store the coordinate errors in a column $e$, so the product error is
the row $e^\top W_{\mathrm{blk}}$.  We write
$g^{\mathrm{fix}}=W_{\mathrm{blk}}^\top e^{\mathrm{fix}}\in\R^p$ for the contribution
forced by fixed coordinates.  In the static regime, $e$ is one column of the
error in a reusable weight factor $W\in\R^{K\times n}$ and is coordinated
for a distribution of input rows (\Cref{sec:disc-static}).  We measure
squared Euclidean output error because its expectation is the product MSE
used for calibration; summing over activation rows gives squared Frobenius
error for a full product.

\paragraph{Running example.}
We will reuse a four-coordinate dynamic instance to trace the construction.
Take $K=4$, a one-column block $W_{\mathrm{blk}}=(1,1,1,1)^\top$, unit steps,
no fixed error, and $\theta=(0.2,0.3,0.6,0.8)^\top$.  For any binary choice
$\chi$, the scalar product error is
\[
  (1,1,1,1)^\top\!\cdot(\chi-\theta)
  =\sum_{k=1}^4\chi_k-1.9.
\]
\Cref{sec:disc-dynamic} follows this instance through the binary
formulation, the null-space reduction, and the conditional-expectation
completion.

\subsection{The independent-decision references}

Independent stochastic rounding chooses the upper level with probability
$\theta_k$ and the lower level otherwise.  This makes each active scalar
error unbiased while preserving the adjacent-level feasible set.

\begin{lemma}[Same-grid stochastic-rounding reference]
\label{lem:disc-bernoulli}
For the feasible sets in \eqref{eq:disc-scalar-error}, independently draw
$\chi_k\sim\operatorname{Bernoulli}(\theta_k)$ for every $k\in U$ and keep
fixed coordinates at their forced values.  Then
$\E[d_k(\chi_k-\theta_k)]=0$ and
\begin{equation}
  \E\norm*{
    g^{\mathrm{fix}}+
    \sum_{k\in U}d_k(\chi_k-\theta_k)
      (W_{\mathrm{blk}})_{k,:}^\top
  }_2^2
  =
  \norm{g^{\mathrm{fix}}}_2^2+
  \sum_{k\in U}d_k^2\theta_k(1-\theta_k)
    \norm*{(W_{\mathrm{blk}})_{k,:}}_2^2 .
  \label{eq:disc-bernoulli-reference}
\end{equation}
For a common step $\Delta$ and offset fractions averaged uniformly over $[0,1]$,
the factor $\theta(1-\theta)$ averages to $1/6$.
\end{lemma}

\begin{proof}
Each active increment has mean zero and variance
$d_k^2\theta_k(1-\theta_k)$.  The zero means remove the cross terms with
$g^{\mathrm{fix}}$, and independence removes the cross terms between distinct
active increments.  Expanding the squared norm therefore gives
\eqref{eq:disc-bernoulli-reference}.  Finally,
$\int_0^1\theta(1-\theta)\,d\theta=1/6$.
\end{proof}

A second reference, \emph{subtractive dither}, supplies a scale independent
of the observed offset fractions.  It lets us compare the certificate
across inputs and block sizes.  Add an independent
$u_k\sim\mathcal U[-\Delta_k/2,\Delta_k/2]$ before uniform quantization,
then subtract the same $u_k$ from the reconstruction.  When the dithered
input stays within the quantizer's range, the resulting error is uniform
and independent of the input \cite{graystockham,grayneuhoff}.  The shifted
reconstruction can lie off the original grid, so same-grid Bernoulli
rounding remains the reference for admissible randomized choices.

\needspace{6\baselineskip}
\begin{lemma}[Unclipped subtractive-dither reference]
\label{lem:disc-dither}
Let $e=(e_1,\ldots,e_K)^\top$ contain independent quantization errors produced by
subtractive dither with steps $\Delta_1,\ldots,\Delta_K$, and assume that no
dithered value clips.  Then $e_k$ is uniform on $[-\Delta_k/2,\Delta_k/2]$,
independent of the input, and
\begin{equation}
  \E\!\norm*{e^\top W_{\mathrm{blk}}}_2^2
  =\frac{1}{12}\sum_{k=1}^K
    \Delta_k^2\norm*{(W_{\mathrm{blk}})_{k,:}}_2^2 .
  \label{eq:disc-dither-reference}
\end{equation}
For a common step $\Delta$, the right-hand side is
$(\Delta^2/12)\norm{W_{\mathrm{blk}}}_F^2$.
\end{lemma}

\begin{proof}
Unclipped subtractive dither gives
$\E e_k=0$ and $\E e_k^2=\Delta_k^2/12$
\cite{graystockham,grayneuhoff}.  After the squared norm is expanded,
independence and zero means remove all cross terms, leaving
\[
  \sum_k \E e_k^2
  \norm*{(W_{\mathrm{blk}})_{k,:}}_2^2 .
\]
Substitution yields \eqref{eq:disc-dither-reference}.
\end{proof}

\paragraph{How to read the comparisons.}
The binary optimum is the smallest squared error over admissible choices;
a method's realized error is what its chosen pattern incurs.  Our
certificate bounds the latter relative to the former
(\Cref{thm:disc-additive}).  Same-grid stochastic rounding averages over
admissible patterns, so its expectation is at least the binary optimum.
Dither uses shifted reconstructions and has no such ordering
(\Cref{fig:disc-quantities}).  In the running example, the optimum is
$0.1^2=0.01$, whereas the dither expectation is $4/12$.
Coordination can therefore beat this noise scale.  With clipping, we keep
the forced offset in the objective and use dither only as a normalization;
\Cref{lem:disc-dither} describes the unclipped noise law.

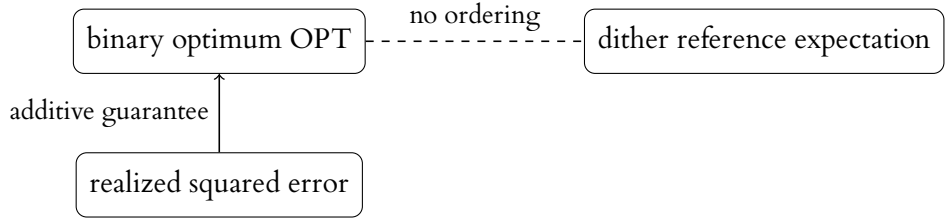
\begin{figure}[!tbp]
  \centering
  \begin{tikzpicture}[
    qty/.style={draw, rounded corners, align=center, inner sep=5pt,
      minimum height=0.85cm},
    flow/.style={->, line width=0.6pt},
    nolink/.style={line width=0.6pt, dashed}
  ]
    \node[qty] (opt) at (-0.4,0.95)
      {binary optimum $\mathrm{OPT}$};
    \node[qty] (alg) at (-0.4,-0.95)
      {realized squared error};
    \node[qty] (ref) at (6.8,0.95)
      {dither reference expectation};
    \draw[flow] (alg) -- node[left, align=right, font=\small]
      {additive guarantee} (opt);
    \draw[nolink] (opt) -- node[above, font=\small]
      {no ordering} (ref);
  \end{tikzpicture}
  \caption{The three comparison quantities.  The constructive guarantee of
  \Cref{thm:disc-additive} is additive relative to the binary optimum.
  Same-grid stochastic rounding has expectation at least the binary optimum;
  the dither expectation is unordered with both
  (\Cref{sec:disc-setting}).}
  \label{fig:disc-quantities}
\end{figure}

\subsection{Notation}

\begin{table}[H]
  \centering
  \small
  \renewcommand{\arraystretch}{1.08}
  \begin{tabular}{@{}p{0.19\linewidth}p{0.74\linewidth}@{}}
    \toprule
    Symbol & Meaning \\
    \midrule
    $K,p,n$ & Shared inner dimension, dynamic block width, and number of
      columns in the reusable static factor. \\
    $\ell,u,d_k,\Delta$ & Adjacent lower and upper levels, their
      coordinate-dependent step size, and a common step when one is assumed. \\
    $\theta,\chi$ & Observed offset fractions and binary lower/upper
      choices. \\
    $U,F,D$ & Active and fixed coordinate sets, and the embedding of active
      step sizes into the full inner dimension. \\
    $e^{\mathrm{fix}},g^{\mathrm{fix}}$ & Errors forced by saturation or exact
      representation, and their image through the weight block. \\
    $W_{\mathrm{blk}},V,v_k$ & Weight block, its step-weighted version with
      rows $v_k^\top$, where $v_k=d_k(W_{\mathrm{blk}})_{k,:}^\top$. \\
    $r,R,K_{\mathrm{eff}}$ & Rank of $V$, its largest row norm
      $R=\max_k\norm{v_k}_2$, and effective row count
      $K_{\mathrm{eff}}=\norm{V}_F^2/R^2$ for $V\ne0$. \\
    \bottomrule
  \end{tabular}
  \caption{Notation for the fixed-grid setup and the dynamic results;
    symbols for the static and hardness sections are defined where they
    first appear.}
  \label{tab:disc-notation}
\end{table}

We now coordinate these choices for one observed activation row and one
weight block.

\section{Dynamic product-aware rounding}
\label{sec:disc-dynamic}

Rounding errors can cancel after multiplication by a weight block.  A dynamic
rounder chooses the up/down decisions together to exploit that cancellation.
We express their combined error as a binary least-squares objective, then
use low rank to settle all but a few decisions without changing the output.

\subsection{The binary least-squares objective}
\label{sec:disc-binary-objective}

Choosing the upper level instead of the lower one changes the product by
the corresponding weight row times the grid step size.  We aim to balance these row
contributions.  Fix one observed activation row and
a weight block
$W_{\mathrm{blk}}\in\R^{K\times p}$.  Use the affine representation
\eqref{eq:disc-affine-error} for the row's rounding error, and define
\[
  g^{\mathrm{fix}}=W_{\mathrm{blk}}^\top e^{\mathrm{fix}}
  \quad\text{and}\quad
  v_k=d_k(W_{\mathrm{blk}})_{k,:}^\top\in\R^p,\qquad k\in U.
\]
Let $V\in\R^{|U|\times p}$ have rows $v_k^\top$.  The vector $g^{\mathrm{fix}}$
includes every error forced by saturation; exact grid values contribute
zero.  For binary choices $\chi\in\{0,1\}^{|U|}$, the squared product error
is
\begin{equation}
  \mathcal L(\chi)
  =
  \norm*{
    g^{\mathrm{fix}}+
    \sum_{k\in U}(\chi_k-\theta_k)v_k
  }_2^2 ,
  \label{eq:disc-objective}
\end{equation}
and the binary optimum is
\begin{equation}
  \mathrm{OPT}_{\mathrm{dyn}}(V,\theta,g^{\mathrm{fix}})
  =
  \min_{\chi\in\{0,1\}^{|U|}}\mathcal L(\chi).
  \label{eq:disc-optdyn}
\end{equation}

The affine representation \eqref{eq:disc-affine-error} gives
$W_{\mathrm{blk}}^\top e=g^{\mathrm{fix}}+\sum_{k\in U}(\chi_k-\theta_k)v_k$, the transpose of the output-error row $e^\top W_{\mathrm{blk}}$.
Each binary vector specifies one admissible rounding pattern, so
\eqref{eq:disc-optdyn} is exactly the smallest error on these fixed grids.

In the running example from \Cref{sec:disc-setting}, all four rows of $V$
equal the scalar one, so $r=\rank(V)=1$ and
\[
  \mathrm{OPT}_{\mathrm{dyn}}(V,\theta,0)
  =\min_{\chi\in\{0,1\}^4}
    \left(\sum_{k=1}^4\chi_k-1.9\right)^2
  =0.01.
\]
Every binary vector with two upper-rounding decisions attains this optimum.

For a matrix of activation rows, the squared Frobenius product error for a
fixed output block is the sum of the row losses
\eqref{eq:disc-objective}, so the constructions below apply independently
to each row.  Different output blocks may require different rounded copies of the same
row; \Cref{sec:disc-discussion} quantifies the cost of those copies.

\subsection{A deterministic additive approximation}
\label{sec:disc-ce}

Whenever more than $r=\rank(V)$ variables remain fractional, their row
contributions are linearly dependent.  We can move those variables together
until another reaches an endpoint, while their output changes cancel.
That is how low rank reduces the number of unresolved decisions.

First allow each decision to range continuously from zero to one.  Values
strictly between the endpoints are \emph{fractional}.  Write the product
image as $\pi(z)=V^\top z$ and solve the convex quadratic program
\begin{equation}
  x^0\in\operatorname*{arg\,min}_{z\in[0,1]^{|U|}}
  \bigl\|g^{\mathrm{fix}}+V^\top(z-\theta)\bigr\|_2^2.
  \label{eq:disc-relaxation}
\end{equation}
Geometrically, this projects the target
$V^\top\theta-g^{\mathrm{fix}}$ onto the image of the cube under $\pi$.
The minimum is zero exactly when that image contains the target.
The reduction below moves along $\ker V^\top$, so it preserves $\pi(z)$
and the objective attained by the continuous solution.

\begin{lemma}[Null-space reduction to at most $r$ fractional coordinates]
\label{lem:disc-null-walk}
Let $V\in\R^{|U|\times p}$ have rows $v_k^\top$ and rank $r$.  There is a
deterministic polynomial-time procedure that starts at any
$x^{\mathrm{init}}\in[0,1]^{|U|}$, preserves
$V^\top x=V^\top x^{\mathrm{init}}$, and terminates at
$x^\star\in[0,1]^{|U|}$ with at most $r$ fractional coordinates.
\end{lemma}

\begin{proof}
Let $J$ be the fractional coordinates of the current point $x$.  While
$|J|>r$, the columns of $V_J^\top$ are linearly dependent.  Choose a nonzero
$h$ supported on $J$ with $V_J^\top h=0$, and move from $x$ in either feasible
direction along $h$ until one coordinate reaches $0$ or $1$.  The move
preserves $V^\top x$, keeps $x$ in the cube, and strictly shrinks $J$.
The procedure therefore terminates with a set $S$ satisfying $|S|\le r$ and
$V^\top x^\star=V^\top x^{\mathrm{init}}$.  Gaussian elimination supplies
each null direction in polynomial time.
\end{proof}

At the endpoint, fixing every integral coordinate leaves a face of the cube
whose $|S|$ free coordinates are the unresolved decisions.  We call it the
\emph{reached face}.  Its $2^{|S|}$ vertices can be enumerated to find the best
completion on that face; finding the best vertex of the full cube would
require searching across faces too.  For the running example, one endpoint is
$x^\star=(0,0,1,0.9)^\top$.  It preserves
$(1,1,1,1)\cdot x^\star=(1,1,1,1)\cdot\theta=1.9$ and leaves only the
fourth coordinate fractional.  Its reached face is therefore the cube edge
obtained by fixing the first three coordinates to $(0,0,1)$.
\Cref{fig:disc-null-space-geometry} illustrates the same reduction in a
three-coordinate rank-one example.

\begin{figure}[p]
  \centering
  \includegraphics[width=\linewidth]{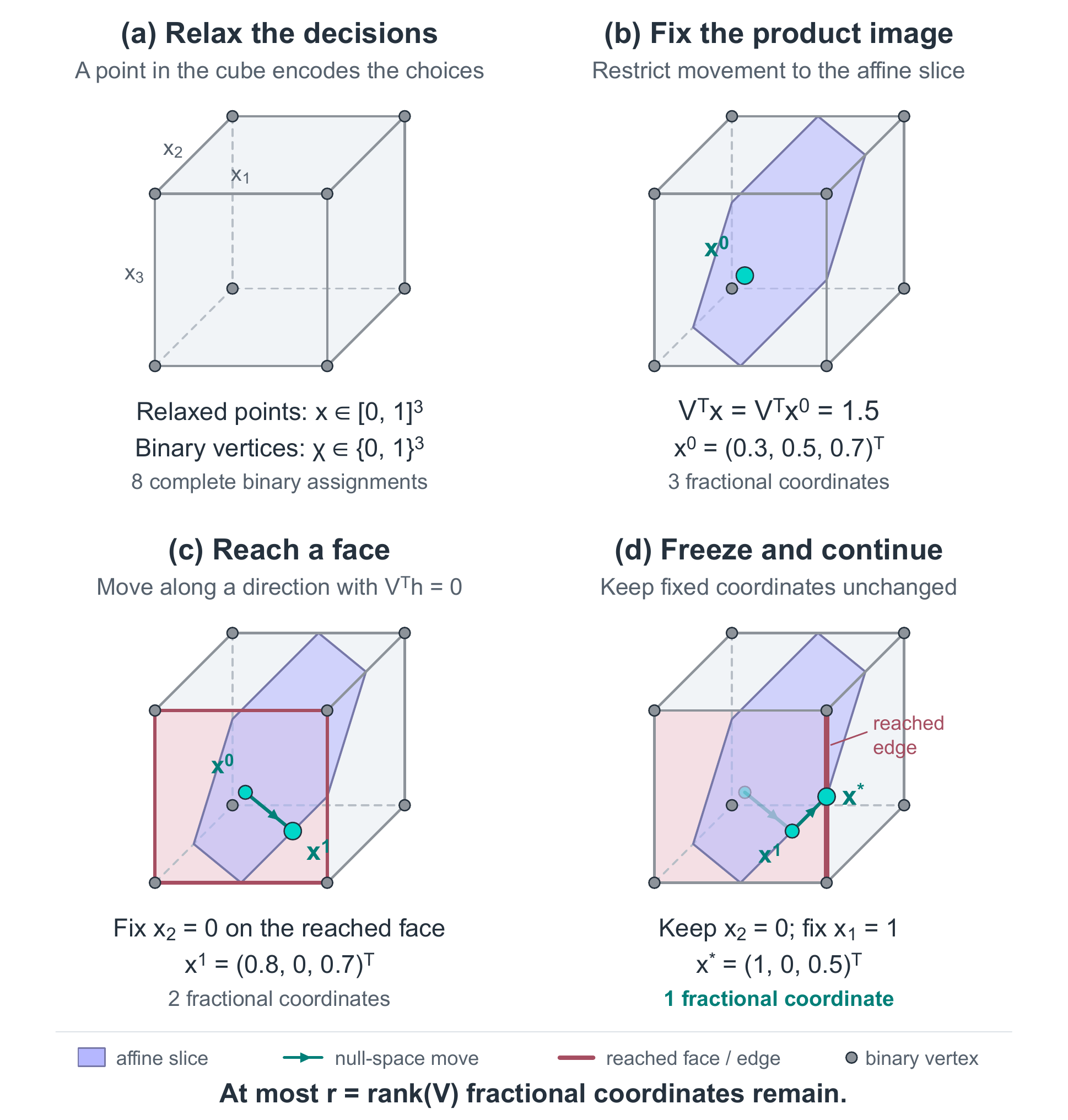}
  \caption{Null-space reduction preserves the relaxed product.  This
  three-coordinate example uses $V=(1,1,1)^\top$, so the affine slice is
  $x_1+x_2+x_3=1.5$.  Starting at $x^0=(0.3,0.5,0.7)^\top$, the first move
  reaches $x^1=(0.8,0,0.7)^\top$ and fixes $x_2=0$ on the highlighted face.
  The second move keeps $x_2=0$ and fixes $x_1=1$, reaching
  $x^\star=(1,0,0.5)^\top$ on the highlighted edge.  Both moves preserve
  $V^\top x$, and the endpoint has one fractional coordinate, matching
  $\rank(V)=1$.}
  \label{fig:disc-null-space-geometry}
\end{figure}

Only the fractional coordinates can now add rounding error.  To finish,
imagine rounding them independently, with $x_k^\star$ as the probability of
choosing the upper level.  Try both choices for the first coordinate and
keep the one with the smaller expected final error.  Leave the later choices
random during this comparison.  At least one choice is no worse than the
weighted average, so repeating the comparison fixes every coordinate without
increasing the initial expectation.  This is the \emph{method of conditional
expectations}.

\Needspace{11\baselineskip}
\begin{lemma}[Conditional-expectation completion]
\label{lem:disc-ce}
Let $x^\star\in[0,1]^{|U|}$ have fractional set $S$.  A deterministic
procedure, running in $O((|U|+|S|)p)$ arithmetic operations, returns
$\chi\in\{0,1\}^{|U|}$ that agrees with $x^\star$ off $S$ and satisfies
\begin{equation}
  \mathcal L(\chi)
  \le
  \mathcal L(x^\star)+
  \sum_{k\in S}x_k^\star(1-x_k^\star)\norm{v_k}_2^2
  \le
  \mathcal L(x^\star)+\frac14\sum_{k\in S}\norm{v_k}_2^2 .
  \label{eq:disc-ce-general}
\end{equation}
\end{lemma}

\begin{proof}
Write
$g'=g^{\mathrm{fix}}+V^\top(x^\star-\theta)$.  If each $k\in S$ is rounded
independently with $\Pr(\chi_k=1)=x_k^\star$, then the increments
$\chi_k-x_k^\star$ are independent and have mean zero.  Expanding the square
therefore gives
\[
  \E\norm*{g'+V_S^\top(\chi_S-x_S^\star)}_2^2
  =
  \norm{g'}_2^2+
  \sum_{k\in S}x_k^\star(1-x_k^\star)\norm{v_k}_2^2 .
\]
Make the choices deterministic by fixing the coordinates of $S$ one at a time
in a fixed order.  At coordinate $k$, evaluate the two conditional
expectations obtained from $\chi_k=0$ and $\chi_k=1$, leaving later
coordinates random, and choose the smaller.  A running product-error sum
evaluates each alternative in $O(p)$ operations; the remaining variance
terms are common to both alternatives.  No choice increases the conditional
expectation, proving the first inequality in \eqref{eq:disc-ce-general}.
The bound $x(1-x)\le1/4$ proves the second.
\end{proof}

At the running-example endpoint $x^\star=(0,0,1,0.9)^\top$, the two choices
for the last coordinate give squared errors $(-0.9)^2=0.81$ and
$(0.1)^2=0.01$.  Conditional expectation selects the upper endpoint and
therefore attains the binary optimum $0.01$ computed above.

The two lemmas give the certificate.  The reduction leaves at most $r$
variance terms in \eqref{eq:disc-ce-general}; each is at most one quarter of
the largest squared row norm.  Define $R=\max_{k\in U}\|v_k\|_2$, with
$R=0$ when $U$ is empty.

\needspace{8\baselineskip}
\begin{theorem}[Dynamic additive approximation]
\label{thm:disc-additive}
Solve \eqref{eq:disc-relaxation}, apply the reduction of
\Cref{lem:disc-null-walk}, and complete the remaining coordinates by
\Cref{lem:disc-ce}.  The returned binary point satisfies
\begin{equation}
  \mathcal L(\chi)
  \le \mathcal L(x^0)+\frac14rR^2
  \le \mathrm{OPT}_{\mathrm{dyn}}(V,\theta,g^{\mathrm{fix}})
       +\frac14rR^2.
  \label{eq:disc-additive}
\end{equation}
\end{theorem}

\begin{proof}
The reduction preserves the product image, so
$\mathcal L(x^\star)=\mathcal L(x^0)$.
Apply \eqref{eq:disc-ce-general} and use $|S|\le r$,
$x_k^\star(1-x_k^\star)\le1/4$, and $\|v_k\|_2\le R$.
Finally, the cube contains every binary feasible point, so its minimum is
at most $\mathrm{OPT}_{\mathrm{dyn}}$.
\end{proof}

An $\eta$-accurate solution of the convex quadratic program adds $\eta$ to
this guarantee.  It can be computed in time polynomial in the encoded input
size and $\log(1/\eta)$ \cite{kozlov1980convex}; Gaussian elimination supplies
the null directions in polynomial time.  Appendix~\ref{sec:disc-algorithms} gives the
procedures and a bound for measured floating-point drift.

\paragraph{Why the starting point matters under clipping.}
Without clipping, the observed fractional target $\theta$ already attains
zero continuous loss, so it can replace $x^0$.  With clipping, starting at
$\theta$ preserves $\pi(\theta)$ throughout the reduction.  At that endpoint the active
coordinates contribute zero net correction to the forced error
$g^{\mathrm{fix}}$.  Optimizing the continuous start instead lets them
compensate that error before any decisions are finalized.

\begin{corollary}[Target start]
\label{cor:disc-theta-start}
Start the reduction instead at $\theta$.  Then $|S|\le r$,
$\mathcal L(x^\star)=\norm{g^{\mathrm{fix}}}_2^2$, and
\begin{equation}
  \mathcal L(\chi)
  \le
  \norm{g^{\mathrm{fix}}}_2^2
  +\frac14rR^2 .
  \label{eq:disc-ce-rank}
\end{equation}
\end{corollary}

\begin{proof}
Invariance gives $V^\top(x^\star-\theta)=0$, so
$\mathcal L(x^\star)=\norm{g^{\mathrm{fix}}}_2^2$.  Substitute into
\eqref{eq:disc-ce-general} with $|S|\le r$ and
$\norm{v_k}_2\le R$.
\end{proof}

The clipping experiment in \Cref{sec:disc-exp-affine} compares these two
starts while keeping the reduction and completion procedures fixed.

\subsection{Comparing the certificate with independent-rounding scales}
\label{sec:disc-dynamic-comparison}

Low rank controls how many decisions remain; the largest row norm controls
their cost.  To compare the resulting certificate with dither, suppose all
$K$ coordinates are active, with common step size $\Delta$ and zero fixed
offset.  Then $V=\Delta W_{\mathrm{blk}}$.  For $V\ne0$, define the
\emph{effective row count}
\begin{equation}
  K_{\mathrm{eff}}
  =
  \frac{\norm{V}_F^2}{R^2},
  \qquad 1\le K_{\mathrm{eff}}\le K .
  \label{eq:disc-effective-rows}
\end{equation}
This count divides total row energy by the largest row energy.  It can
exceed rank: $K$ identical nonzero rows have rank one and
$K_{\mathrm{eff}}=K$.  A dominant row brings the count close to one.
Writing $\beta_{\mathrm{rms}}$ and $\beta_{\max}$ for the root-mean-square
and largest row norms of $W_{\mathrm{blk}}$ gives
$K_{\mathrm{eff}}=K(\beta_{\mathrm{rms}}/\beta_{\max})^2$.
The related participation ratio is
$(\sum_k\|v_k\|_2^2)^2/\sum_k\|v_k\|_2^4$; the maximum in
\eqref{eq:disc-effective-rows} is the normalization required by the certificate.
The dither reference energy \eqref{eq:disc-dither-reference} and the
certified conditional-expectation bound from \eqref{eq:disc-ce-rank} are
\[
  D_{\mathrm{dither}}
  =\frac{1}{12}\norm{V}_F^2,
  \qquad
  B_{\mathrm{CE}}
  =\frac14rR^2,
\]
so their ratio is
\begin{equation}
  \frac{B_{\mathrm{CE}}}{D_{\mathrm{dither}}}
  =
  \frac{3r}{K_{\mathrm{eff}}} .
  \label{eq:disc-conditional-factor}
\end{equation}
Thus, whenever row energy is spread across more than $3r$ effective rows, the
certificate itself lies below the dither reference: coordinated rounding is
guaranteed to beat independent noise of that scale.  The certificate is
worst-case, while the dither reference is an expectation over independent
noise.

The condition is sufficient, not necessary.  Take two identical scalar rows
$v_1=v_2=1$ and $\theta=(1/2,1/2)^\top$.  Then $K_{\mathrm{eff}}=2<3r$, yet
one upper and one lower choice give zero error, compared with
$D_{\mathrm{dither}}=1/6$.  The experiments measure how far the realized
error lies below the certificate
(\Cref{sec:disc-exp-dynamic,sec:disc-exp-imbalance}).

For the same-grid stochastic reference
\eqref{eq:disc-bernoulli-reference}, the denominator is instead
\[
  \Delta^2\sum_k\theta_k(1-\theta_k)
  \norm*{(W_{\mathrm{blk}})_{k,:}}_2^2 .
\]
This reference depends on the observed offset fractions $\theta$, whereas
dither provides an input-independent scale for the threshold above.

\paragraph{An additive scheme by enumerating heavy rows.}
The maximum row norm in the certificate suggests a direct refinement:
fix the few largest rows by enumeration, then apply the certificate to the
remaining rows.  This trades running time for accuracy.

\begin{corollary}[Heavy-row enumeration]
\label{cor:disc-heavy-enumeration}
For $V\ne0$, let $r=\rank(V)$ and
$\mathcal D(V)=\|V\|_F^2/12$.  For every $\varepsilon>0$ and $\eta\in(0,1)$, enumeration
followed by the dynamic construction returns a binary point with
\[
  \mathcal L(\chi)
  \le \mathrm{OPT}_{\mathrm{dyn}}(V,\theta,g^{\mathrm{fix}})
      +\varepsilon\mathcal D(V)+\eta
\]
in $2^{\lceil3r/\varepsilon\rceil}$ times polynomial time in the encoded
input size and $\log(1/\eta)$.  The target and fixed offset are arbitrary.
\end{corollary}

\begin{proof}
Set $\tau=\varepsilon\|V\|_F^2/(3r)$ and enumerate every binary choice on
$H=\{k:\|v_k\|_2^2>\tau\}$.  Its cardinality is less than
$3r/\varepsilon$.  In each branch, absorb those choices into the fixed
offset and apply \Cref{thm:disc-additive} to the remaining rows.
Their rank is at most $r$ and their squared norms are at most $\tau$.
The branch agreeing with a global optimum therefore has error at most
$\mathrm{OPT}_{\mathrm{dyn}}+r\tau/4+\eta
=\mathrm{OPT}_{\mathrm{dyn}}+\varepsilon\mathcal D(V)+\eta$.
Return the best branch.  If every row is fixed, the branch is evaluated
directly; when $V=0$, every choice already has the same loss.
\end{proof}

\paragraph{Extensions toward reusable rounding.}
The conditional-expectation construction controls one block-specific
Euclidean error.  Reusing one rounded row across several output blocks
requires control of its error in several directions at once.
Appendix~\ref{sec:disc-extensions} develops a randomized alternative with this
broader control; the reuse question remains open (\Cref{op:disc-reuse}).

\section{Static reusable rounding}
\label{sec:disc-static}

Static rounding chooses one weight matrix for future inputs.  Its loss must
therefore average over those inputs.  Whether it should include the mean
output error depends on a deployment choice: can the output bias be changed
after the weights are rounded?

\subsection{The reusable objective}

Let $x\in\R^{1\times K}$ be a random activation row.  For a fixed rounded
matrix $\widehat W=W+E_W$, write $E_W=[e_1,\ldots,e_n]$, so the column
$e_j\in\R^K$ produces output error $xe_j$ at coordinate $j$.  Define the
\emph{uncentered second moment}
\begin{equation}
  M=\E[x^\top x]\succeq0.
  \label{eq:disc-second-moment}
\end{equation}
The row convention makes $x^\top x$ a $K\times K$ matrix; $xx^\top$ would
be a scalar.  Its entry $M_{ik}=\E[x_i x_k]$ weights the interaction between
weight errors at input coordinates $i$ and $k$.  Expanding the output square
gives $\E(xe_j)^2=\sum_{i,k}(e_j)_iM_{ik}(e_j)_k$.

A nonzero input mean creates a systematic output error.  With
$\mu=\E[x]$, its contribution separates from the variance:
\begin{equation}
  \E(xe_j)^2
  =e_j^\top(M-\mu^\top\mu)e_j+(\mu e_j)^2.
  \label{eq:disc-mean-variance}
\end{equation}
Centering drops the last term.  A bias correction can remove it from the
actual output too; \Cref{rem:disc-bias} gives that correction.
First, the fixed-bias problem separates across output columns.

\needspace{9\baselineskip}
\begin{theorem}[Static product-loss objective]
\label{thm:disc-static}
Fix every scale, clipping interval, and quantization grid for a weight
factor $W\in\R^{K\times n}$ before choosing its rounding decisions.  Let
$\widehat W=W+E_W$, with $E_W=[e_1,\ldots,e_n]$, and suppose the remaining
feasible set factors by column as
$\mathcal R_1\times\cdots\times\mathcal R_n$.
Then
\begin{equation}
  \min_{\widehat W}\E_x\norm{x(\widehat W-W)}_2^2
  =
  \sum_{j=1}^n\min_{e_j\in\mathcal R_j} e_j^\top M e_j .
  \label{eq:disc-static-separation}
\end{equation}
\end{theorem}

\begin{proof}
The product error separates across output coordinates:
\[
  \norm{xE_W}_2^2=\sum_{j=1}^n(xe_j)^2.
\]
Taking expectations and using \eqref{eq:disc-second-moment} gives
\[
  \E_x\norm{xE_W}_2^2
  =\sum_{j=1}^n e_j^\top \E[x^\top x]e_j
  =\sum_{j=1}^n e_j^\top M e_j .
\]
The product form of the feasible set lets the minimum pass
through the sum, proving \eqref{eq:disc-static-separation}.
\end{proof}

The quadratic identity holds for every rounded matrix, including when $M$
is singular.  A vector in $\ker M$ represents a weight error invisible to
the input distribution.  Independent minimization by column additionally
requires the product feasible set in the theorem: shared scales, clipping
bounds, or transforms must be fixed before that minimization.

\begin{remark}[When centering is correct]
\label{rem:disc-bias}
If the output bias can be recalibrated after rounding, each error column
admits the correction
\[
  c_j^\star=-\mu e_j,
  \qquad
  \min_{c_j}\E(xe_j+c_j)^2
  =e_j^\top(M-\mu^\top\mu)e_j.
\]
Thus centered covariance measures the loss with optimal bias correction;
$M$ measures the loss with fixed or absent bias.  The two coincide when
$\mu=0$.  Centered covariance alone does not predict which rounded weights
will perform worse: the omitted $(\mu e_j)^2$ determines that comparison,
and the Digits diagnostic measures its size in \Cref{sec:disc-exp-static}.
\end{remark}

\subsection{Relation to activation-weighted reconstruction}

Post-training weight quantizers choose rounded weights by minimizing the
change in a layer's outputs on calibration activations.  This is
\emph{activation-weighted reconstruction}.  Let $X\in\R^{N\times K}$
contain those activation rows and set $M=X^\top X/N$.  Then
\begin{equation}
  \frac1N\norm{X(\widehat W-W)}_F^2
  =\sum_{j=1}^n e_j^\top Me_j,
  \label{eq:disc-reconstruction}
\end{equation}
the empirical form of the static loss.

For comparison with GPTQ \cite{gptq}, define
$W_{\mathrm{GPTQ}}=W^\top$ and
$\widehat W_{\mathrm{GPTQ}}=\widehat W^\top$.
Equation \eqref{eq:disc-reconstruction} then equals
$N^{-1}\| (\widehat W_{\mathrm{GPTQ}}-W_{\mathrm{GPTQ}})X^\top\|_F^2$.
GPTQ minimizes this reconstruction loss through sequential error compensation.
The identity here specifies the exact metric and the condition under which
bias recalibration replaces it by centered covariance.  Empirically, the
mean is $\bar x=N^{-1}\mathbf1^\top X$ and the correction is
$c_j=-\bar x e_j$.

\paragraph{Constructive consequence.}
Factoring $M=\Phi\Phi^\top$ reduces the static objective to the same
Euclidean form used by the dynamic construction.  For column $j$, write its
affine adjacent-level representation as
\[
  e_j=e_j^{\mathrm{fix}}+D_j(\chi_j-\theta_j),
  \qquad \chi_j\in\{0,1\}^{|U_j|}.
\]
If $M=\Phi\Phi^\top$ has rank $r$, then
$e_j^\top Me_j=\norm{\Phi^\top e_j}_2^2$.  Applying the null-space reduction
to the rows of $D_j^\top\Phi$ preserves
$\Phi^\top e_j$ while leaving at most $\rank(D_j^\top\Phi)\le r$ fractional
decisions.  The fixed vector $\Phi^\top e_j^{\mathrm{fix}}$ enters the objective;
the adjustable rows determine the null space.
Starting from the continuous minimum and completing by conditional
expectation therefore returns a column with
\[
 e_j^\top M e_j
 \le \mathrm{OPT}_{j,\mathrm{relax}}+\frac14r_jR_j^2,
 \qquad
 r_j=\rank(D_j^\top\Phi),\quad
 R_j=\max_{k\in U_j}d_k\norm{\Phi_{k,:}}_2.
\]
Here $\mathrm{OPT}_{j,\mathrm{relax}}$ is the minimum over the continuous
box.  When there are no active choices, we set $R_j=0$, so the additive
term is zero.
This is \Cref{thm:disc-additive} applied to the static column.
The Digits diagnostic uses the same one-flip descent for each calibration
metric (\Cref{sec:disc-exp-static}).

\paragraph{Extensions.}
A low-rank approximation of $M$ retains its strongest input directions and
discards the weaker directions, which form the \emph{spectral tail}.  This
approximation
reduces the dimension of the rounding problem, but error in the discarded directions
still contributes to the full product loss.  Appendix~\ref{sec:disc-extensions}
bounds that contribution in terms of the same rounding output's tail error
second moment and reports a diagnostic for the required condition.

\section{Complexity of exact product-aware rounding}
\label{sec:disc-hardness}

Low rank reduces the number of decisions left for completion; finding the
best rounding pattern nonetheless remains hard.  At rank one, zero error can
encode \textsc{Partition}: split a list of integers into two equal-sum
groups.  With unrestricted rank, the loss can encode \textsc{Max-Cut}:
split a graph's vertices so that as many edges as possible cross the split.

We use rational inputs encoded in binary and measure time in the Turing
model.  The rank-one reduction is \emph{weakly} NP-complete because algorithms
polynomial in the coefficient values can take exponential time in their
bit length.  The graph reduction is \emph{strongly} NP-complete: its
coefficients have polynomially bounded values.

\subsection{Weak NP-completeness at rank one}

Exact zero-error decision is weakly NP-complete even at rank one.  Consider
the zero-threshold decision version of
\eqref{eq:disc-optdyn} with no fixed offset, one output coordinate, unit step size,
and target $\theta=\frac12\mathbf1$.  The input coefficients are positive
integers encoded in binary.  Fixing the positive common step size to one loses no
generality because a common scale does not affect the zero test.

\begin{theorem}[Rank-one zero-error rounding is weakly NP-complete]
\label{thm:disc-partition}
Given positive binary-encoded integers $v_1,\ldots,v_K$, deciding whether
\begin{equation}
  \min_{\chi\in\{0,1\}^K}
  \left|
    \sum_{k=1}^K(\chi_k-\tfrac12)v_k
  \right|^2
  =0
  \label{eq:disc-rank-one-decision}
\end{equation}
is weakly NP-complete.  The corresponding dynamic block has rank one.  The
same claim holds for static rounding with the rank-one metric
$M=vv^\top$.
\end{theorem}

\begin{proof}
Set $s_k=2\chi_k-1\in\{-1,1\}$; switching to $\pm1$ signs turns the objective
into a signed sum, the form that connects it to \textsc{Partition}.  The
objective in
\eqref{eq:disc-rank-one-decision} is exactly
\begin{equation}
  \frac{1}{4}
  \left(\sum_{k=1}^K s_kv_k\right)^2 .
  \label{eq:disc-partition-objective}
\end{equation}
It vanishes if and only if the integers can be partitioned into two equal-sum
subsets.  We map a \textsc{Partition} instance to the one-column block
$W_{\mathrm{blk}}=v$ in polynomial time; that block has rank one whenever
$v\ne0$.  We can verify equality for a proposed sign assignment in polynomial
time, so
the decision problem is in NP.  Because \textsc{Partition} is weakly NP-complete
\cite{gareyjohnson1979}, this equivalence proves NP-completeness.
For the static realization, use the deterministic input row $x=v^\top$, so
$M=\E[x^\top x]=vv^\top$, and take
$e=\chi-\frac12\mathbf1$.  Then
\[
  e^\top Me
  =\frac{1}{4}
    \left(\sum_k s_kv_k\right)^2,
\]
which is the same objective.
\end{proof}

Dynamic programming tracks attainable sums up to $\sum_kv_k$ and solves
the problem in time polynomial in that value.  This is a
\emph{pseudo-polynomial} algorithm; the sum can be exponentially large in
the bit length of the input.

At rank one, the reduction in \Cref{lem:disc-null-walk} leaves at most one
fractional coordinate.  Completing that face requires just two comparisons.
Finding a face containing a globally optimal vertex is the hard part.

\paragraph{Consequences for rank-dependent algorithms.}
Rank remains useful for approximation: the additive certificate and
heavy-row enumeration in \Cref{thm:disc-additive,cor:disc-heavy-enumeration}
give polynomial-time guarantees at every fixed rank.  Exact optimization is
what rank alone cannot buy.  A method polynomial at every fixed rank would
already solve the hard rank-one case, so if $\ell$ is the encoded input
length, both $\ell^{f(r)}$ time (called XP) and $f(r)\ell^{O(1)}$ time
(fixed-parameter tractability, or FPT) for exact zero-error decision would
imply $\mathrm{P}=\mathrm{NP}$.

\subsection{Strong NP-completeness of the PSD threshold problem}

The next reduction uses the bounded entries of an unweighted graph.
Assigning each vertex an up/down choice will turn its edges into output
coordinates.

We count an error of one for
each edge whose endpoints receive the same choice, and zero when they
receive opposite choices.  Minimizing the total edge penalty maximizes the cut.
The matrix that collects these edge penalties is the \emph{signless
Laplacian}.  Throughout this subsection, $Q_G$ denotes the signless Laplacian
for a graph
$G$: the sum of its diagonal degree matrix and its adjacency matrix.  If $H$
is the unoriented vertex--edge incidence matrix of $G$, then $Q_G=HH^\top$.

\begin{theorem}[Rounding in an unrestricted PSD quadratic loss is strongly
NP-complete]
\label{thm:disc-maxcut}
Given an unweighted graph $G$ and an integer $C$, deciding whether
\[
 \min_{\chi\in\{0,1\}^K}
 (\chi-\tfrac12\mathbf1)^\top Q_G(\chi-\tfrac12\mathbf1)
 \le |E|-C
\]
is strongly NP-complete, where $K$ is the number of vertices and the unit
step size is fixed.  Consequently, threshold minimization of $e^\top Me$ for
$e=\chi-\tfrac12\mathbf1$, $\chi\in\{0,1\}^K$, is strongly NP-complete
for general $M\succeq0$.  The same reduction gives strong NP-hardness for
the dynamic problem with unrestricted block rank.
\end{theorem}

\begin{proof}
Let $G=(V_G,E)$ be an unweighted graph with adjacency matrix $A_G$ and
diagonal degree matrix $\mathrm{Deg}_G$, and set
\[
  Q_G=\mathrm{Deg}_G+A_G,
\]
as above (the ordinary Laplacian has $-A_G$ in place of $+A_G$).  Let
$H\in\{0,1\}^{|V_G|\times|E|}$ be the unoriented vertex--edge incidence
matrix, with $H_{ve}=1$ exactly when vertex $v$ is an endpoint of edge $e$.
Then $Q_G=HH^\top\succeq0$.  For
$s=2\chi-\mathbf1\in\{-1,1\}^{|V_G|}$,
\begin{equation}
  \frac14s^\top Q_Gs
  =\frac14\sum_{\{i,j\}\in E}(s_i+s_j)^2
  =|E|-\operatorname{cut}_G(s).
  \label{eq:disc-signless-cut}
\end{equation}
Consequently,
\[
  e^\top Q_Ge
  =|E|-\operatorname{cut}_G(s).
\]
There is a cut of size at least $C$ if and only if this objective is at most
$|E|-C$.  This mapping is a polynomial reduction from unweighted
\textsc{Max-Cut}, whose simple-graph decision form is NP-complete
\cite{gareyjohnsonstockmeyer1976,gareyjohnson1979}; Karp's earlier
formulation allowed edge weights \cite{karp1972}.  The problem is also in NP:
a proposed $\chi$ specifies a cut, and the entries of $Q_G$ and the threshold
have polynomially bounded magnitude, so verifying
$e^\top Q_Ge\le|E|-C$ takes polynomial time.  The polynomial magnitude
bound also establishes strong NP-completeness.

The metric has a calibration interpretation.  If we sample an edge
$f$ uniformly and take $x=h_f^\top$, where $h_f$ is its incidence column
in $H$, we obtain $\E[x^\top x]=Q_G/|E|$; the positive factor $1/|E|$
rescales the threshold to $(|E|-C)/|E|$.  For the dynamic formulation, choose
$W_{\mathrm{blk}}=H$, so the squared product error is
$e^\top HH^\top e=e^\top Q_Ge$.
\end{proof}

\paragraph{Why an additive guarantee is appropriate.}
A finite-factor multiplicative guarantee must return zero whenever the
optimum is zero: multiplying zero by any finite factor still gives zero.
Such a polynomial-time algorithm would therefore decide the rank-one
\textsc{Partition} instances above, which is impossible unless
$\mathrm{P}=\mathrm{NP}$.  An additive guarantee leaves a controlled
absolute error and avoids the requirement to return zero at zero optimum.
\Cref{thm:disc-additive}
provides an additive guarantee, with a penalty governed by rank
and the largest step-weighted row norm.

The dependence on rank and approximation accuracy, rather than exact
optimization at fixed rank, is the remaining algorithmic question
(\Cref{sec:disc-discussion}).

\section{Experiments}
\label{sec:disc-experiments}

The experiments test four consequences of the theory: cancellation in
low-rank blocks, sensitivity to row-energy concentration, compensation for
clipping offsets, and the bias convention in static rounding.  Controlled
synthetic blocks isolate the first three effects; held-out Digits inputs
test the fourth under a common optimization procedure.

\newif\ifdiscresults
\IfFileExists{tex/discrepancy/generated/results.tex}{%
  \discresultstrue
  % Generated by experiments/reproduce.py; do not edit.
\providecommand{\ExperimentSeed}{20260723}
\providecommand{\DiscDynamicExactTable}{% Generated by experiments/reproduce.py; do not edit.
\begin{tabular}{rrrrrrrrr}
\toprule
$r$ & Exact & RTN & Bern. & Nearest & CE & Cert. & Diag.\ oracle & Face gap \\
\midrule
$1$ & $0.046$ & $0.515$ & $2.010$ & $0.046$ & $0.046$ & $0.188$ & $0.046$ & $0$ \\
$2$ & $4.796\times 10^{-5}$ & $0.655$ & $2.013$ & $0.088$ & $0.088$ & $0.375$ & $0.088$ & $0.084$ \\
$4$ & $0.009$ & $0.940$ & $2.007$ & $0.188$ & $0.165$ & $0.750$ & $0.144$ & $0.136$ \\
$8$ & $0.130$ & $0.837$ & $1.971$ & $0.407$ & $0.347$ & $1.500$ & $0.255$ & $0.120$ \\
\bottomrule
\end{tabular}
}
\providecommand{\DiscDynamicTable}{% Generated by experiments/reproduce.py; do not edit.
\begin{tabular}{rrrrrrr}
\toprule
$K$ & $r$ & RTN & Bern. & Nearest & CE [IQR] & Diag.\ oracle \\
\midrule
$64$ & $4$ & $0.794$ & $2.025$ & $0.045$ & $0.042\,[0.038,0.045]$ & $0.041$ \\
$64$ & $16$ & $0.929$ & $2.000$ & $0.244$ & $0.166\,[0.162,0.171]$ & $0.114$ \\
$256$ & $4$ & $0.827$ & $1.991$ & $0.012$ & $0.011\,[0.009,0.012]$ & $0.010$ \\
$256$ & $16$ & $0.969$ & $2.002$ & $0.055$ & $0.040\,[0.038,0.041]$ & $0.027$ \\
$1024$ & $4$ & $0.900$ & $2.003$ & $0.003$ & $0.003\,[0.002,0.003]$ & $0.003$ \\
$1024$ & $16$ & $0.899$ & $2.004$ & $0.014$ & $0.010\,[0.010,0.011]$ & $0.007$ \\
\bottomrule
\end{tabular}
}
\providecommand{\DiscDynamicRuntimeTable}{% Generated by experiments/reproduce.py; do not edit.
\begin{tabular}{lrr}
\toprule
Method & Error/dither & Median time (ms) \\
\midrule
RTN & $0.899$ & $0.002$ \\
Analytic Bernoulli evaluation & $2.004$ & $0.015$ \\
Nearest completion & $0.014$ & $75.851$ \\
CE completion & $0.010$ & $75.916$ \\
Reached-face diagnostic oracle & $0.007$ & $76.961$ \\
\bottomrule
\end{tabular}
}
\providecommand{\DiscSensitivityTable}{% Generated by experiments/reproduce.py; do not edit.
\begin{tabular}{rrrrrrrr}
\toprule
$\alpha$ & $\beta_{\rm rms}/\beta_{\max}$ & Cert. & RTN & Bern. & Nearest & CE & Diag.\ oracle \\
\midrule
$0$ & $1.000$ & $0.047$ & $0.903$ & $2.000$ & $0.011$ & $0.010$ & $0.010$ \\
$0.500$ & $0.155$ & $1.959$ & $0.759$ & $2.042$ & $0.041$ & $0.036$ & $0.035$ \\
$1.000$ & $0.080$ & $7.312$ & $0.768$ & $2.244$ & $0.479$ & $0.379$ & $0.306$ \\
$2.000$ & $0.065$ & $11.087$ & $0.758$ & $2.147$ & $0.710$ & $0.604$ & $0.493$ \\
\bottomrule
\end{tabular}
}
\providecommand{\DiscAffineTable}{% Generated by experiments/reproduce.py; do not edit.
\begin{tabular}{rrrrrrr}
\toprule
$\rho$ (\%) & $\theta$/dither & QP/dither & $\theta$ (raw) & QP (raw) & Ratio & QP time (ms) \\
\midrule
$0$ & $0.010$ & $0.009$ & $0.215$ & $0.199$ & $1.1$ & $13.421$ \\
$5$ & $0.181$ & $0.010$ & $3.853$ & $0.213$ & $18.1$ & $12.686$ \\
$10$ & $0.450$ & $0.010$ & $9.597$ & $0.208$ & $46.1$ & $12.131$ \\
\bottomrule
\end{tabular}
}
\providecommand{\DiscStaticTable}{% Generated by experiments/reproduce.py; do not edit.
\begin{tabular}{rrlr}
\toprule
Bits & Cal. & Method & Held-out/RTN [95\% RI] \\
\midrule
$3$ & 128 & RTN & $1.000\,[1.000,1.000]$ \\
$3$ & 128 & RTN + bias & $0.860\,[0.830,0.871]$ \\
$3$ & 128 & Centered, fixed bias & $1.264\,[0.933,2.035]$ \\
$3$ & 128 & Centered + bias & $0.756\,[0.598,0.837]$ \\
$3$ & 128 & Uncentered & $0.842\,[0.692,0.929]$ \\
$3$ & full & RTN & $1.000\,[1.000,1.000]$ \\
$3$ & full & RTN + bias & $0.858\,[0.829,0.870]$ \\
$3$ & full & Centered, fixed bias & $1.250\,[0.884,2.027]$ \\
$3$ & full & Centered + bias & $0.730\,[0.575,0.812]$ \\
$3$ & full & Uncentered & $0.799\,[0.682,0.905]$ \\
$4$ & 128 & RTN & $1.000\,[1.000,1.000]$ \\
$4$ & 128 & RTN + bias & $0.616\,[0.375,0.738]$ \\
$4$ & 128 & Centered, fixed bias & $1.148\,[0.867,1.753]$ \\
$4$ & 128 & Centered + bias & $0.378\,[0.250,0.472]$ \\
$4$ & 128 & Uncentered & $0.478\,[0.319,0.581]$ \\
$4$ & full & RTN & $1.000\,[1.000,1.000]$ \\
$4$ & full & RTN + bias & $0.615\,[0.374,0.737]$ \\
$4$ & full & Centered, fixed bias & $1.260\,[0.832,1.835]$ \\
$4$ & full & Centered + bias & $0.351\,[0.237,0.462]$ \\
$4$ & full & Uncentered & $0.467\,[0.317,0.562]$ \\
\bottomrule
\end{tabular}
}
\providecommand{\DiscrepancyDynamicCEMilliseconds}{75.916}
\providecommand{\DiscrepancyDynamicRTNMilliseconds}{0.002}
\providecommand{\DiscrepancyDynamicExactK}{16}
\providecommand{\DiscrepancyDynamicTrials}{25}
\providecommand{\DiscrepancyDynamicBlocks}{10}
\providecommand{\DiscrepancyDynamicTrialsPerConfiguration}{250}
\providecommand{\DiscrepancyDynamicLargestK}{1024}
\providecommand{\DiscrepancyDynamicLargestRank}{16}
\providecommand{\DiscrepancyDynamicBernoulliKThousandRankSixteenRatio}{2.004}
\providecommand{\DiscrepancyDynamicRTNKThousandRankSixteenRatio}{0.899}
\providecommand{\DiscrepancyDynamicArbitraryKThousandRankSixteenRatio}{0.014}
\providecommand{\DiscrepancyDynamicCEKThousandRankSixteenRatio}{0.010}
\providecommand{\DiscrepancyDynamicNullKThousandRankSixteenRatio}{0.007}
\providecommand{\DiscrepancySensitivityBalancedRatio}{1.000}
\providecommand{\DiscrepancySensitivityBalancedCertificate}{0.047}
\providecommand{\DiscrepancySensitivityBalancedCERatio}{0.010}
\providecommand{\DiscrepancySensitivityImbalancedRatio}{0.065}
\providecommand{\DiscrepancySensitivityImbalancedCertificate}{11.087}
\providecommand{\DiscrepancySensitivityImbalancedCERatio}{0.604}
\providecommand{\DiscrepancyAffineK}{256}
\providecommand{\DiscrepancyAffineRank}{4}
\providecommand{\DiscrepancyAffineThetaTenPercentRatio}{0.450}
\providecommand{\DiscrepancyAffineBoxTenPercentRatio}{0.010}
\providecommand{\DiscrepancyAffineFivePercentImprovement}{18.1}
\providecommand{\DiscrepancyAffineTenPercentImprovement}{46.1}
\providecommand{\DiscrepancyMaximumWalkInvariantError}{2.004\times 10^{-14}}
\providecommand{\DiscrepancyMaximumAffineWalkInvariantError}{4.595\times 10^{-15}}
\providecommand{\DiscrepancyMaximumCEBoundViolation}{0}
\providecommand{\DiscrepancyMaximumExactViolation}{4.718\times 10^{-16}}
\providecommand{\DiscrepancyStaticSplits}{5}
\providecommand{\DiscrepancyStaticCalibrationSubsets}{20}
\providecommand{\DiscrepancyStaticCalibrationSize}{128}
\providecommand{\DiscrepancyStaticFullCalibrationSize}{359}
\providecommand{\DiscrepancyStaticTestSize}{360}
\providecommand{\DiscrepancyStaticFitSize}{1078}
\providecommand{\DiscrepancyStaticCalibrationPoolSize}{359}
\providecommand{\DiscrepancyStaticRidge}{0.001}
\providecommand{\DiscrepancyStaticThreeBitCenteredRatio}{1.264}
\providecommand{\DiscrepancyStaticThreeBitCenteredBiasRatio}{0.756}
\providecommand{\DiscrepancyStaticThreeBitUncenteredRatio}{0.842}
\providecommand{\DiscrepancyStaticThreeBitUncenteredFlips}{30}
\providecommand{\DiscrepancyStaticFourBitCenteredRatio}{1.148}
\providecommand{\DiscrepancyStaticFourBitCenteredBiasRatio}{0.378}
\providecommand{\DiscrepancyStaticFourBitUncenteredRatio}{0.478}
\providecommand{\DiscrepancyStaticFourBitUncenteredFlips}{48}
\providecommand{\DiscrepancyStaticThreeBitTransferCorrelation}{0.155}
\providecommand{\DiscrepancyStaticFourBitTransferCorrelation}{0.607}
\providecommand{\DiscrepancyStaticThreeBitPooledTransferCorrelation}{0.897}
\providecommand{\DiscrepancyStaticFourBitPooledTransferCorrelation}{0.962}
\providecommand{\DiscrepancyStaticThreeBitTransferCorrelationBySplit}{0.155}
\providecommand{\DiscrepancyStaticFourBitTransferCorrelationBySplit}{0.607}
\providecommand{\DiscrepancyStaticThreeBitFullCenteredRatio}{1.250}
\providecommand{\DiscrepancyStaticThreeBitFullCenteredBiasRatio}{0.730}
\providecommand{\DiscrepancyStaticThreeBitFullUncenteredRatio}{0.799}
\providecommand{\DiscrepancyStaticFourBitFullCenteredRatio}{1.260}
\providecommand{\DiscrepancyStaticFourBitFullCenteredBiasRatio}{0.351}
\providecommand{\DiscrepancyStaticFourBitFullUncenteredRatio}{0.467}
\providecommand{\DiscrepancyStaticThreeBitFullTransferCorrelation}{0.962}
\providecommand{\DiscrepancyStaticFourBitFullTransferCorrelation}{0.992}
\providecommand{\DiscrepancyStaticMaximumIdentityError}{1.665\times 10^{-16}}
\providecommand{\DiscrepancyStaticMaximumCenteredBiasObjectiveGap}{3.816\times 10^{-17}}
}{%
  \discresultsfalse
}
\newif\ifgswtailresults
\IfFileExists{tex/discrepancy/generated/gsw_tail_results.tex}{%
  \gswtailresultstrue
  % Generated by experiments/reproduce.py; do not edit.
\providecommand{\ExperimentSeed}{20260723}
\providecommand{\DiscrepancyGSWTailK}{32}
\providecommand{\DiscrepancyGSWTailBlocks}{4}
\providecommand{\DiscrepancyGSWTailDraws}{1024}
\providecommand{\DiscrepancyGSWTailMinimumMedianRatio}{4.391}
\providecommand{\DiscrepancyGSWTailMaximumMedianRatio}{8.942}
\providecommand{\DiscrepancyGSWTailMinimumAnalyticBernoulliRatio}{2.855}
\providecommand{\DiscrepancyGSWTailMaximumAnalyticBernoulliRatio}{2.999}
\providecommand{\DiscrepancyGSWTailMinimumMedianOverAnalyticBernoulli}{1.471}
\providecommand{\DiscrepancyGSWTailMaximumMedianOverAnalyticBernoulli}{3.077}
\providecommand{\DiscrepancyGSWTailMinimumMedianOverSampledBernoulli}{1.303}
\providecommand{\DiscrepancyGSWTailMaximumMedianOverSampledBernoulli}{2.842}
\providecommand{\DiscrepancyGSWTailMedianHalfToFullChangePercent}{3.1}
\providecommand{\DiscrepancyGSWTailMaximumHalfToFullChangePercent}{14.1}
\providecommand{\DiscrepancyGSWTailMaximumMarginalError}{0.049}
\providecommand{\DiscrepancyGSWTailMaximumOrthogonalityError}{1.258\times 10^{-11}}
}{%
  \gswtailresultsfalse
}

\subsection{Dynamic low-rank blocks}
\label{sec:disc-exp-dynamic}

Many rounding decisions can cancel in a product with few output directions.
The first study measures how that cancellation changes with row count, rank,
and output width.  It also separates the quality of the final rounding step
from the best solution on the full binary cube.

\paragraph{Design.}
For each balanced configuration, the output block has shape $K\times p$ with
$p=r$.  Its row directions are independent standard-Gaussian draws normalized
to unit norm; the common adjacent-level step size is one.
Offset fractions are independent uniform random variables on $[0,1]$.  We use
\ifdiscresults\DiscrepancyDynamicBlocks\else10\fi{} independent
blocks per $(K,r)$ and
\ifdiscresults\DiscrepancyDynamicTrials\else25\fi{} independent fractional
targets per block.  The exact study sets $K=16$ and
$r\in\{1,2,4,8\}$.  The scalable study uses every $(K,r)$ pair with
$K\in\{64,128,256,512,1024\}$ and $r\in\{1,2,4,8,16\}$; the printed table
reports a representative subset, and the artifact contains the full grid.
A separate fixed-width sweep holds $K=256$ and $p=32$ while varying
$r\in\{1,4,16\}$.  It embeds each $K\times r$ factor through a seeded
orthonormal $p\times r$ matrix, preserving row norms and product errors.
These wide blocks separate rank from output width.
\paragraph{Methods and metrics.}
We first apply the common null-space reduction, which leaves at most $r$
fractional coordinates.  We then compare two polynomial completion rules:
\emph{nearest completion} rounds those coordinates to their nearest endpoints,
whereas \emph{CE completion} applies the conditional-expectation
completion of \Cref{lem:disc-ce}.  The tables use the labels Nearest and CE.

To distinguish completion quality from global optimality, the
\emph{reached-face diagnostic oracle} exhaustively searches the face reached
by the reduction.  A face is
the subset of $[0,1]^{|U|}$ obtained by fixing some coordinates to $0$ or $1$;
the reached face has dimension at most $r$ and therefore at most $2^r$
vertices.  Its $2^r$ search gives the optimum on that reached face.  At
$K=16$, full $2^{16}$ enumeration additionally supplies the global optimum,
allowing us to measure the gap between the two.
\ifdiscresults
The ``face gap'' column of \Cref{tab:disc-dynamic-exact} reports the pooled
median of (diagnostic-oracle error minus global exact error), normalized by
dither.
\fi

The independent baselines are round-to-nearest (RTN) and the analytic
same-grid Bernoulli expectation from \Cref{lem:disc-bernoulli}.
Thus RTN and Bernoulli show independent decisions, Nearest and CE compare
two ways of finishing the same reduction, and the diagnostic oracle
measures what the best completion on that reached face could achieve.

We divide every squared product error by the analytic unclipped-dither
expectation, so a ratio below one means less error than this dither
reference.  This denominator puts different block sizes on the same scale.
The tables report the median across ten block-level medians, each
summarizing 25 targets.  Brackets denote the interquartile range across these
medians.  The raw pooled-target statistics remain in the artifact.

\Cref{fig:disc-dynamic-scaling} summarizes the two comparisons: its left
panel shows CE error as more rows share each output direction, and its right
panel compares the completion rules.  Both panels use dither-normalized
error and logarithmic axes.

\ifdiscresults
The balanced sweep shows decreasing CE error as row count grows at fixed
rank.  For $r=4$, the median ratio falls from $0.042$ at $K=64$ to $0.011$
at $K=256$ and $0.003$ at $K=1024$.  For $r=16$, the corresponding ratios
are $0.166$, $0.040$, and $0.010$.  The certificate scales as $3r/K$ in
these balanced blocks, and the measured errors lie well below it: at
$K=16,r=4$, the certificate is $0.75$ dither units and the
CE median is $0.165$, a factor of $4.5$ apart
(\Cref{tab:disc-dynamic-exact}).

\begin{table}[tbp]
  \centering
  \scriptsize
  % Generated by experiments/reproduce.py; do not edit.
\begin{tabular}{rrrrrrrrr}
\toprule
$r$ & Exact & RTN & Bern. & Nearest & CE & Cert. & Diag.\ oracle & Face gap \\
\midrule
$1$ & $0.046$ & $0.515$ & $2.010$ & $0.046$ & $0.046$ & $0.188$ & $0.046$ & $0$ \\
$2$ & $4.796\times 10^{-5}$ & $0.655$ & $2.013$ & $0.088$ & $0.088$ & $0.375$ & $0.088$ & $0.084$ \\
$4$ & $0.009$ & $0.940$ & $2.007$ & $0.188$ & $0.165$ & $0.750$ & $0.144$ & $0.136$ \\
$8$ & $0.130$ & $0.837$ & $1.971$ & $0.407$ & $0.347$ & $1.500$ & $0.255$ & $0.120$ \\
\bottomrule
\end{tabular}

  \caption{Small-instance comparison at
  $K=\DiscrepancyDynamicExactK$.  Method entries are medians across
  \DiscrepancyDynamicBlocks{} block-level medians, with
  \DiscrepancyDynamicTrials{} targets per block.  ``Face gap'' is the pooled
  median of (diagnostic-oracle error minus global exact error), normalized by
  dither.  The positive gaps for $r=2,4,8$ measure the difference between the best
  reached-face completion and the global optimum.  Headings abbreviate round-to-nearest
  (RTN), same-grid Bernoulli (Bern.), conditional expectation (CE), and
  reached-face diagnostic oracle (Diag.).  ``Cert.'' is $3r/K$.
  Each rank uses independent block draws, so the Exact column compares
  methods within a row; its cross-row variation is not a rank trend.}
  \label{tab:disc-dynamic-exact}
\end{table}

\begin{table}[tbp]
  \centering
  \scriptsize
  % Generated by experiments/reproduce.py; do not edit.
\begin{tabular}{rrrrrrr}
\toprule
$K$ & $r$ & RTN & Bern. & Nearest & CE [IQR] & Diag.\ oracle \\
\midrule
$64$ & $4$ & $0.794$ & $2.025$ & $0.045$ & $0.042\,[0.038,0.045]$ & $0.041$ \\
$64$ & $16$ & $0.929$ & $2.000$ & $0.244$ & $0.166\,[0.162,0.171]$ & $0.114$ \\
$256$ & $4$ & $0.827$ & $1.991$ & $0.012$ & $0.011\,[0.009,0.012]$ & $0.010$ \\
$256$ & $16$ & $0.969$ & $2.002$ & $0.055$ & $0.040\,[0.038,0.041]$ & $0.027$ \\
$1024$ & $4$ & $0.900$ & $2.003$ & $0.003$ & $0.003\,[0.002,0.003]$ & $0.003$ \\
$1024$ & $16$ & $0.899$ & $2.004$ & $0.014$ & $0.010\,[0.010,0.011]$ & $0.007$ \\
\bottomrule
\end{tabular}

  \caption{Balanced scalable comparison for a representative subset of the
  swept $(K,r)$ grid; the artifact reports all 25 configurations.  Entries are
  medians across ten block-level medians; conditional-expectation (CE)
  completion also shows its interquartile range (IQR).  Each block contributes
  25 independent fractional targets.  The Bernoulli column is an analytic
  expectation at each observed fractional target, whereas dither is the denominator and
  therefore equals one.  RTN denotes round-to-nearest and Diag.\ the
  reached-face diagnostic oracle, which is exponential in $r$.}
  \label{tab:disc-dynamic-scalable}
\end{table}

\begin{table}[tbp]
  \centering
  \small
  % Generated by experiments/reproduce.py; do not edit.
\begin{tabular}{lrr}
\toprule
Method & Error/dither & Median time (ms) \\
\midrule
RTN & $0.899$ & $0.002$ \\
Analytic Bernoulli evaluation & $2.004$ & $0.015$ \\
Nearest completion & $0.014$ & $75.851$ \\
CE completion & $0.010$ & $75.916$ \\
Reached-face diagnostic oracle & $0.007$ & $76.961$ \\
\bottomrule
\end{tabular}

  \caption{Representative end-to-end time per fractional target at
  $K=1024,r=16$.  Losses are medians across block medians; times are pooled
  medians and include null-space reduction and residual rounding.  The
  Bernoulli timing measures closed-form evaluation of its expected loss.  The reached-face diagnostic-oracle time is
  specific to $r=16$ and retains exponential worst-case scaling.  RTN denotes
  round-to-nearest and CE conditional expectation.}
  \label{tab:disc-dynamic-runtime}
\end{table}

The fixed-width results in \Cref{tab:disc-dynamic-wide} show cancellation
in rank-deficient blocks.  Every row has 32 output entries, while
the rank varies independently.  The table gives raw squared errors as well
as ratios; here one dither unit is $256/12=21.333$.

\begin{table}[tbp]
  \centering
  \small
  \begin{tabular}{rrrrrrr}
\toprule
$K$ & $p$ & $r$ & RTN/dither & CE/dither & Cert./dither & CE (raw) \\
\midrule
$256$ & $32$ & $1$ & $0.447$ & $0.003$ & $0.012$ & $0.061$ \\
$256$ & $32$ & $4$ & $0.926$ & $0.011$ & $0.047$ & $0.236$ \\
$256$ & $32$ & $16$ & $0.964$ & $0.040$ & $0.188$ & $0.862$ \\
\bottomrule
\end{tabular}

  \caption{Fixed-width balanced blocks: ten independent blocks and 25 targets
  per block.  Entries are medians across block medians.  ``Cert.'' is the
  deterministic bound, and raw CE error is the squared Euclidean product
  error before normalization.}
  \label{tab:disc-dynamic-wide}
\end{table}

At $K=1024,r=16$, CE gives
\DiscrepancyDynamicCEKThousandRankSixteenRatio{} dither units, compared with
\DiscrepancyDynamicRTNKThousandRankSixteenRatio{} for RTN.
The measured cost is \(\DiscrepancyDynamicCEMilliseconds\)\,ms per target
(\Cref{tab:disc-dynamic-runtime}).  This timing reflects the Python/NumPy
reference implementation, which recomputes the rank and runs a small SVD at
every reduction step; it establishes the computational cost of the algorithm
rather than the throughput of an optimized kernel.
Appendix~\ref{app:disc-dynamic-protocol} gives the
CPU, BLAS, thread settings, and same-machine product timings.
\fi

\begin{figure}[tbp]
  \centering
  \IfFileExists{figures/discrepancy/dynamic_scaling.pdf}{%
    \includegraphics[width=0.90\linewidth]
      {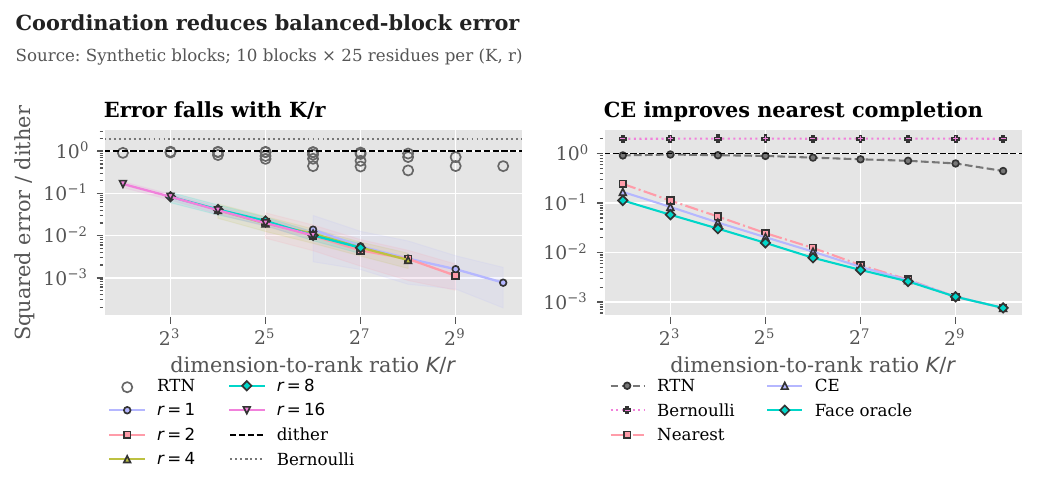}%
  }{}
  \caption{Balanced dynamic rounding.  The left panel shows
  conditional-expectation (CE) completion against
  $K/r$ by rank; colored bands are pooled target-level interquartile ranges.
  The right panel separates round-to-nearest (RTN), analytic Bernoulli, the two polynomial
  completion rules, and the reached-face diagnostic oracle; each right-panel
  point pools all $p=r$ configurations with the displayed ratio.
  Its bands describe pooled heterogeneity across configurations, blocks,
  and targets; the tables report summaries for each configuration.  Tables
  \ref{tab:disc-dynamic-exact}--\ref{tab:disc-dynamic-scalable} report
  cross-block summaries.}
  \label{fig:disc-dynamic-scaling}
\end{figure}

\subsection{Row-norm imbalance}
\label{sec:disc-exp-imbalance}

The fixed-width study isolates rank.  We next change row norms.  This imbalance tests the prediction that $K_{\mathrm{eff}}$, rather than rank alone, controls
the certified bound.  We fix $K=256$, take $r\in\{1,4\}$, and assign row norms
proportional to $j^{-\alpha}$ for
$\alpha\in\{0,0.5,1,2\}$, after a seeded random permutation.  We rescale each
block to unit root-mean-square row norm, keeping the dither denominator fixed
while $\beta_{\mathrm{rms}}/\beta_{\max}$ (and hence $K_{\mathrm{eff}}$) changes.

\ifdiscresults
\Cref{tab:disc-sensitivity} and \Cref{fig:disc-dynamic-stress-tests}(a)
report the sweep.  For $r=4$, moving from $\alpha=0$ to $\alpha=2$ lowers
$\beta_{\mathrm{rms}}/\beta_{\max}$ from
\DiscrepancySensitivityBalancedRatio{} to
\DiscrepancySensitivityImbalancedRatio{}.  The certified CE-to-dither ratio
rises from \DiscrepancySensitivityBalancedCertificate{} to
\DiscrepancySensitivityImbalancedCertificate{}, crossing one as predicted.
The realized block-median ratio for CE completion also degrades, from
\DiscrepancySensitivityBalancedCERatio{} to
\DiscrepancySensitivityImbalancedCERatio{}, while staying below one.  At $\alpha=2$, the certificate therefore fails
to certify a gain that the measured solution attains: the sufficient
$K_{\mathrm{eff}}>3r$ condition is not necessary for cancellation.  For $r=1$, the realized curve is non-monotone: the CE
median ratios are $0.00352$, $0.00141$, $0.000726$, and $0.0502$ at
$\alpha=0,0.5,1,2$, respectively.  Error first decreases and then increases,
although the certified bound weakens monotonically.  The sweep separates two effects: concentration weakens the
certificate, while the realized median also depends on the particular
targets and row directions.

\begin{table}[tbp]
  \centering
  \scriptsize
  \resizebox{\linewidth}{!}{% Generated by experiments/reproduce.py; do not edit.
\begin{tabular}{rrrrrrrr}
\toprule
$\alpha$ & $\beta_{\rm rms}/\beta_{\max}$ & Cert. & RTN & Bern. & Nearest & CE & Diag.\ oracle \\
\midrule
$0$ & $1.000$ & $0.047$ & $0.903$ & $2.000$ & $0.011$ & $0.010$ & $0.010$ \\
$0.500$ & $0.155$ & $1.959$ & $0.759$ & $2.042$ & $0.041$ & $0.036$ & $0.035$ \\
$1.000$ & $0.080$ & $7.312$ & $0.768$ & $2.244$ & $0.479$ & $0.379$ & $0.306$ \\
$2.000$ & $0.065$ & $11.087$ & $0.758$ & $2.147$ & $0.710$ & $0.604$ & $0.493$ \\
\bottomrule
\end{tabular}
}
  \caption{Row-norm sweep at $K=256,r=4$.  ``Cert.'' is the deterministic
  conditional-expectation (CE) certified bound divided by dither; method
  columns are medians across the ten block-level medians.  RTN denotes
  round-to-nearest, Bern.\ same-grid Bernoulli, Diag.\ the reached-face
  diagnostic oracle, and RMS root mean square.  All blocks have unit RMS row
  norm.}
  \label{tab:disc-sensitivity}
\end{table}
\fi

\subsection{Fixed offsets from clipping}
\label{sec:disc-exp-affine}

Clipping forces errors that the rounder cannot change directly, but the
remaining adjustable choices can cancel part of their effect on the
product.  We test whether a continuous optimization that includes these forced
errors
can find the compensation before the final rounding step.

\paragraph{Design.}
To test \Cref{thm:disc-additive}, we fix $K=256$, $r=4$, and
force $\rho\in\{0,0.05,0.10\}$ of the coordinates to clipping endpoints.
Their outward signs are independent Rademacher draws (random signs) and their
overshoots are uniform on $[0.25,1.25]$ grid steps.  We compare two
initializations followed by the same conditional-expectation completion.
The
\emph{fractional-target start} begins the null-space reduction at the observed
targets; the \emph{quadratic-program (QP) start} first solves the continuous
box-constrained least-squares relaxation, whose solution lies in the box
$[0,1]^{|U|}$, and begins the reduction there.  Both errors are divided by the same
full-block dither scale $K/12=21.333$, held fixed across clipping fractions.
The raw errors and the former affine Bernoulli scale are also recorded.

\ifdiscresults
The continuous relaxation cancels the fixed offset to floating-point
precision in every tested instance.  The remaining error comes from
rounding its at most four fractional coordinates.  A start at $\theta$
instead preserves $V^\top\theta$, leaving the fixed offset uncompensated.
This explains the separation in \Cref{tab:disc-affine}:
at ten-percent clipping, raw median squared error falls from $9.597$ to
$0.208$, or from
\DiscrepancyAffineThetaTenPercentRatio{} to
\DiscrepancyAffineBoxTenPercentRatio{} dither units.  The table reports
both scales so the changing offset cannot obscure the error increase.

\begin{table}[H]
  \centering
  \small
  \resizebox{\linewidth}{!}{% Generated by experiments/reproduce.py; do not edit.
\begin{tabular}{rrrrrrr}
\toprule
$\rho$ (\%) & $\theta$/dither & QP/dither & $\theta$ (raw) & QP (raw) & Ratio & QP time (ms) \\
\midrule
$0$ & $0.010$ & $0.009$ & $0.215$ & $0.199$ & $1.1$ & $13.421$ \\
$5$ & $0.181$ & $0.010$ & $3.853$ & $0.213$ & $18.1$ & $12.686$ \\
$10$ & $0.450$ & $0.010$ & $9.597$ & $0.208$ & $46.1$ & $12.131$ \\
\bottomrule
\end{tabular}
}
  \caption{Fixed-offset study.  Error columns are medians across ten
  block-level medians, each based on 25 targets; improvement is their ratio.
  The dither denominator is $256/12$ throughout.  Runtime includes the continuous-relaxation solve.  QP
  denotes the box-constrained quadratic program.}
  \label{tab:disc-affine}
\end{table}
\fi

\begin{figure}[tbp]
  \centering
  \begin{subfigure}[t]{0.485\linewidth}
    \centering
    \IfFileExists{figures/discrepancy/sensitivity_imbalance.pdf}{%
      \includegraphics[width=\linewidth]
        {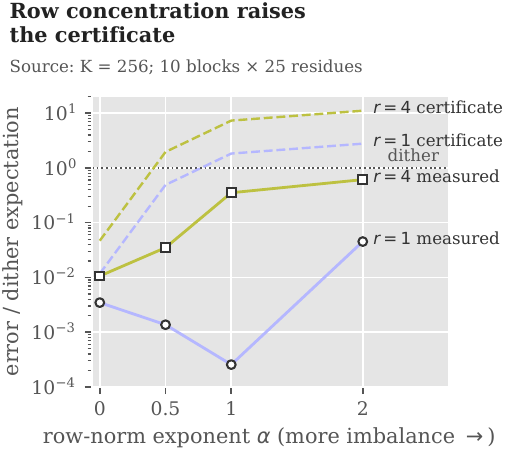}%
    }{}
    \caption{Row-norm imbalance.}
    \label{fig:disc-sensitivity}
  \end{subfigure}\hfill
  \begin{subfigure}[t]{0.485\linewidth}
    \centering
    \IfFileExists{figures/discrepancy/affine_saturation.pdf}{%
      \includegraphics[width=\linewidth]
        {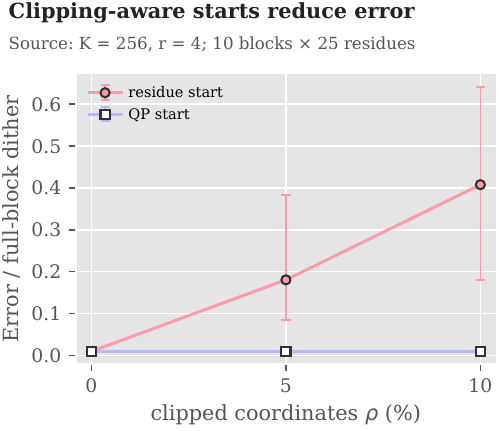}%
    }{}
    \caption{Fixed clipping offsets.}
    \label{fig:disc-affine}
  \end{subfigure}
  \caption{Stress tests for the dynamic-rounding construction.  \emph{(a)}
  Measured conditional-expectation (CE) error and its certified bound under
  increasing row-norm imbalance.  \emph{(b)} Clipping-aware
  quadratic-program (QP) initialization keeps the product error small as the
  fixed-coordinate fraction grows.  Each estimate pools ten blocks and 25
  targets per block; vertical bars in (b) span pooled interquartile ranges.}
  \label{fig:disc-dynamic-stress-tests}
\end{figure}

In the affine study the relaxation cancels the fixed offset exactly, so the
reported improvement isolates the rounding penalty itself.  When cancellation
is incomplete, the same guarantee of \Cref{thm:disc-additive} adds the
relaxed residual to the rounding penalty.

\ifdiscresults
\paragraph{Numerical validation.}
We determine numerical rank by counting singular values above
$\sigma_{\max}(V)\max(|U|,p)\varepsilon_{\mathrm{mach}}$, where
$\varepsilon_{\mathrm{mach}}$ is machine precision, and independently
check that the reduction preserves the product.
The largest preserved-product residual is
\(\DiscrepancyMaximumWalkInvariantError\) in the balanced and imbalance
studies and \(\DiscrepancyMaximumAffineWalkInvariantError\) in the affine
study.  The maximum CE-bound violation is
\(\DiscrepancyMaximumCEBoundViolation\); no violations occur for
diagnostic-oracle dominance or the rank-sized residual condition.  No method beats full
enumeration by more than
\(\DiscrepancyMaximumExactViolation\), which we attribute to floating-point
roundoff.
\fi

\subsection{A mean-term diagnostic on Digits}
\label{sec:disc-exp-static}

A fixed output bias leaves the input mean in the product-error objective.
An adjustable bias can absorb it.  This experiment makes that distinction
visible in a small ridge model on 1,797 public $8\times8$ Digits images,
scaled by $1/16$ \cite{alpaydin1998digits,pedregosa2011scikit}.
It is a diagnostic of the mean contribution, with held-out data used to
measure transfer of the calibration choice.

\paragraph{Design and solver.}
Five stratified splits each fit a ten-output ridge model, reserve a
calibration pool, and hold out evaluation inputs.  At three and four bits,
we compare RTN, centered one-flip descent, and uncentered one-flip descent.
Every descent starts from RTN and uses the same stopping rule.  Both RTN and
the centered endpoint are also evaluated with the correction
$c=-\mu(\widehat W-W)$, using the same calibration mean $\mu$, so both
methods receive the same bias freedom and the paired comparison isolates the
effect of the weight change.
Held-out inputs never choose the correction.
Appendix~\ref{app:disc-static-protocol} specifies the splits, grids, and solver.

\paragraph{Results and mechanism.}
\ifdiscresults
Centered fixed-bias rounding has median held-out/RTN ratios $1.264$,
$1.250$, $1.148$, and $1.260$ across the four bit-width/calibration-size
settings, with bootstrap intervals containing one
(\Cref{tab:disc-static}); the paired split results and arithmetic means
appear in \Cref{tab:disc-static-paired}.  The median ratio exceeds one at
every setting, consistent with the predicted mean-term mechanism.  The
experiment is a diagnostic of that mechanism rather than a population-level
comparison.

Giving both methods the same bias freedom changes the comparison.
With 128 calibration examples, RTN-plus-bias alone reaches $0.860$ and
$0.616$ of uncorrected RTN error at three and four bits.  The paired
centered-plus-bias/RTN-plus-bias ratios are $0.869$ and $0.667$, respectively.
At four bits the raw split-median MSEs are $0.0652$ for RTN, $0.0447$ for
RTN-plus-bias, and $0.0307$ for centered-plus-bias.  Thus both the correction
and the change in rounded weights contribute to the measured reduction.

The mechanism is explicit.  For a weight-error column $e_j$, centering
drops $(\mu e_j)^2$ from the loss.  The raw-pixel input mean makes that
term substantial here, as the per-split decomposition in
Appendix~\ref{app:disc-static-protocol} shows.  Bias correction removes the mean
term on calibration data, and the RTN-plus-bias rows expose how much of
the held-out gain comes from this correction alone.

\begin{table}[tbp]
  \centering
  \small
  % Generated by experiments/reproduce.py; do not edit.
\begin{tabular}{rrlr}
\toprule
Bits & Cal. & Method & Held-out/RTN [95\% RI] \\
\midrule
$3$ & 128 & RTN & $1.000\,[1.000,1.000]$ \\
$3$ & 128 & RTN + bias & $0.860\,[0.830,0.871]$ \\
$3$ & 128 & Centered, fixed bias & $1.264\,[0.933,2.035]$ \\
$3$ & 128 & Centered + bias & $0.756\,[0.598,0.837]$ \\
$3$ & 128 & Uncentered & $0.842\,[0.692,0.929]$ \\
$3$ & full & RTN & $1.000\,[1.000,1.000]$ \\
$3$ & full & RTN + bias & $0.858\,[0.829,0.870]$ \\
$3$ & full & Centered, fixed bias & $1.250\,[0.884,2.027]$ \\
$3$ & full & Centered + bias & $0.730\,[0.575,0.812]$ \\
$3$ & full & Uncentered & $0.799\,[0.682,0.905]$ \\
$4$ & 128 & RTN & $1.000\,[1.000,1.000]$ \\
$4$ & 128 & RTN + bias & $0.616\,[0.375,0.738]$ \\
$4$ & 128 & Centered, fixed bias & $1.148\,[0.867,1.753]$ \\
$4$ & 128 & Centered + bias & $0.378\,[0.250,0.472]$ \\
$4$ & 128 & Uncentered & $0.478\,[0.319,0.581]$ \\
$4$ & full & RTN & $1.000\,[1.000,1.000]$ \\
$4$ & full & RTN + bias & $0.615\,[0.374,0.737]$ \\
$4$ & full & Centered, fixed bias & $1.260\,[0.832,1.835]$ \\
$4$ & full & Centered + bias & $0.351\,[0.237,0.462]$ \\
$4$ & full & Uncentered & $0.467\,[0.317,0.562]$ \\
\bottomrule
\end{tabular}

  \caption{Held-out product mean-squared-error (MSE) ratios on Digits,
  normalized by split-matched round-to-nearest (RTN) at the same precision.
  Entries are medians of \DiscrepancyStaticSplits{} split-level medians; the
  bracketed intervals are bootstrap 95\% resampling intervals (RI) over those
  five observed splits.  For calibration size 128, each split-level median summarizes
  \DiscrepancyStaticCalibrationSubsets{} subsets; ``full'' uses the complete
  calibration pool once per split.  RTN+bias and centered+bias use their
  own errors with the same calibration mean.}
  \label{tab:disc-static}
\end{table}
\fi

\begin{figure}[tbp]
  \centering
  \IfFileExists{figures/discrepancy/static_transfer.pdf}{%
    \includegraphics[width=0.96\linewidth]
      {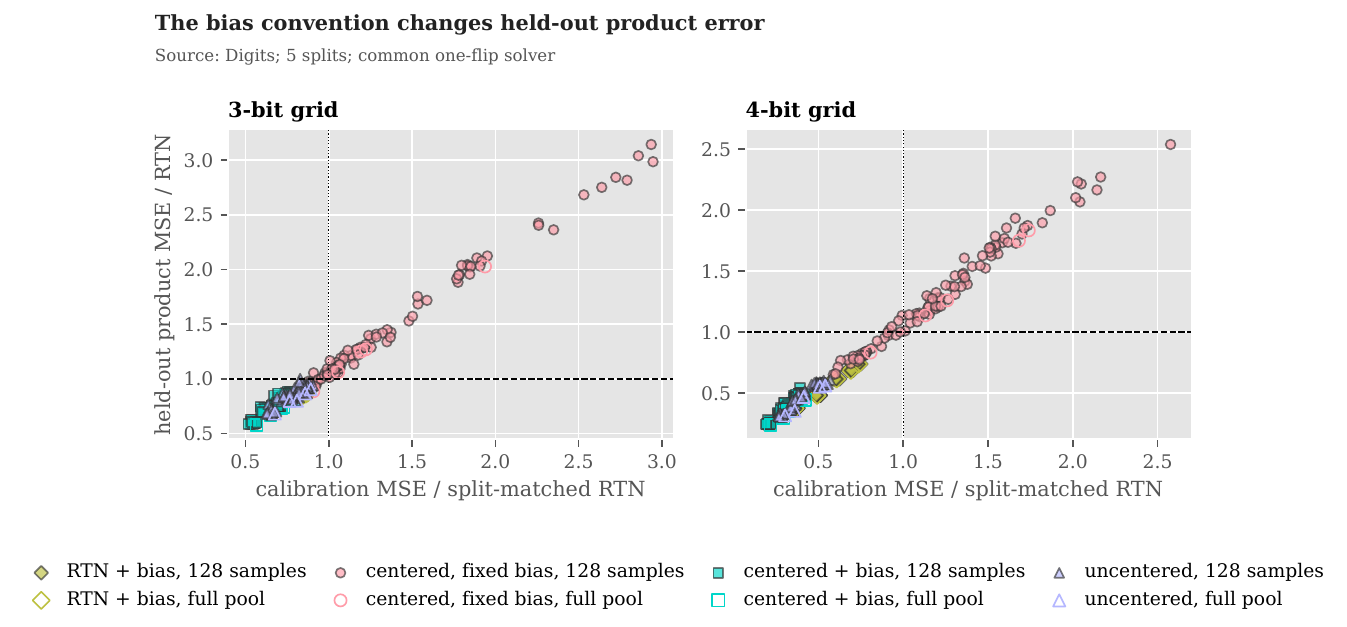}%
  }{}
  \caption{Generalization from calibration to held-out data across five Digits
  splits.  Both axes use the split-matched round-to-nearest (RTN) denominator.
  Marker shapes distinguish RTN-plus-bias, centered fixed-bias,
  centered-plus-bias, and uncentered rounding; filled markers show 128-example subsets and open
  markers show the full calibration pool.}
  \label{fig:disc-static-transfer}
\end{figure}

\ifdiscresults
\paragraph{Numerical validation.}
The direct product and quadratic losses agree to
\(\DiscrepancyStaticMaximumIdentityError\), and the centered-plus-bias
objective identity agrees to
\(\DiscrepancyStaticMaximumCenteredBiasObjectiveGap\).

\fi

\paragraph{Reproducibility.}
The public artifact provides the data, numerical checks, and protocol needed
to reproduce the reported tables and figures:
\url{https://github.com/piyush314/gauge-floors}.

\section{Related work}
\label{sec:disc-related}

Product-aware rounding connects two questions: how to balance vector errors,
and how to preserve a layer's outputs after quantization.
\Cref{tab:disc-positioning} distinguishes their targets and reuse constraints.

\begin{table}[!tbp]
\centering
\footnotesize
\setlength{\tabcolsep}{3pt}
\renewcommand{\arraystretch}{1.15}
\begin{tabular}{@{}p{0.15\linewidth}p{0.19\linewidth}p{0.14\linewidth}p{0.22\linewidth}p{0.23\linewidth}@{}}
\toprule
Method / role & Target & Loss or norm & Rounding choices & Reuse constraint \\
\midrule
Dynamic construction & One observed row--block product & Squared Euclidean & Fixed adjacent levels; forced clipping retained & One activation copy per output block \\
Static identity & Expected linear-layer output & Second-moment quadratic & Fixed columnwise sets & One weight matrix for future inputs \\
AdaRound \cite{adaround} & Calibration layer outputs & Reconstruction loss & Learned adjacent-level decisions & Weights reused across inputs \\
GPTQ \cite{gptq} & Calibration linear outputs & Gram-weighted quadratic & Grid rounding with sequential compensation & Weights reused across inputs \\
Linear discrepancy \cite{lsv} & Worst fractional target & Commonly $\ell_\infty$ & Binary vector & Guarantee over targets \\
Babai \cite{babai1986} & Nearest lattice point & Euclidean distance & Unrestricted integer basis coefficients & Fixed lattice basis; new target per query \\
\bottomrule
\end{tabular}
\caption{Comparison of mathematical targets and feasible sets.  The static
identity specifies a loss; AdaRound and GPTQ supply optimization procedures.
Babai's lattice coefficients require additional constraints to represent the
adjacent-level choices of this paper.}
\label{tab:disc-positioning}
\end{table}

\paragraph{Balancing the product image.}
The dynamic objective is fixed-target Euclidean discrepancy: the observed
offset fractions specify a target, and the rounded entries specify a binary
vector whose image should approach it.  Linear discrepancy instead maximizes
over fractional targets \cite{lsv}.  Our reduction uses linear dependence
to move within the box until another coordinate becomes integral, the
geometry behind iterated-rounding arguments such as Beck--Fiala
\cite{beckfiala1981}.  Conditional expectation then selects endpoints
without increasing the expected squared error
\cite{raghavan1988probabilistic}.  Dependent rounding additionally studies
prescribed marginals and negative correlation \cite{gandhi2006dependent}.

Directional control asks for more than one Euclidean error bound.  For
vectors of bounded norm, the Gaussian-measure result of Banaszczyk gives
signed sums lying in a suitably scaled convex body \cite{banaszczyk}; the
Gram--Schmidt walk constructs such signs \cite{gramschmidt}.
Appendix~\ref{sec:disc-extensions} specializes its directional control to
product errors and tests the spectral-tail condition needed for reuse.

\paragraph{Learning reusable weights.}
Activation-weighted reconstruction coordinates weight decisions using
calibration inputs.  AdaRound optimizes the decisions through a continuous
relaxation, and GPTQ propagates each rounding error to weights still to be
rounded \cite{adaround,gptq}.  Analyses of GPTQ identify its updates with
nearest-plane search under stated orientation and no-clipping assumptions
\cite{gptqgeometry,birnick2025lattice}.  DiscQuant also uses discrepancy,
with guarantees based on approximately low-rank gradient structure
\cite{discquant}.  Our static identity explains which quadratic metric
matches the bias convention of a linear layer; it can therefore guide the
loss used by a reconstruction procedure.

\paragraph{Bounded closest-point search.}
The dynamic problem asks for the point of
$\{V^\top\chi:\chi\in\{0,1\}^{|U|}\}$ nearest to
$V^\top\theta-g^{\mathrm{fix}}$.  For rational data, clearing denominators
embeds this finite set in a rank-$r$ lattice.  Arbitrary real rows need not
generate a discrete lattice.  Even in the rational case, a lattice search
may return an integer combination with coefficients outside $\{0,1\}$,
which is inadmissible for the fixed adjacent-level grids.

Our certificate compares directly with the bounded binary optimum and
measures the additive squared-error cost using rank and the original
step-weighted row norms.  Lattice reduction followed by Babai's
nearest-plane procedure instead approximates unrestricted closest-point
distance in terms of the reduced basis and dimension
\cite{lenstra1982factor,babai1986}; the two guarantees address different
feasible sets.  Constrained tree search can enforce the binary choices at a combinatorial
search cost; sphere decoding offers a related closest-point strategy
\cite{agrell2002closest}.

\paragraph{Changing the quantization geometry.}
Sequential noise shaping feeds earlier errors into later decisions to
redistribute error through a filter \cite{daubechiesdevore2003,floydsteinberg1976}.
Transform methods change the representation before rounding: WUSH designs
block transforms for weights and activations \cite{wush}, and direct
matrix-product quantization co-designs representations and distortion
\cite{ordentlich2024optimal}.  Scalar level-density design offers another
choice before rounding \cite{ang2026scalar}.
For independent dither, product-error accounting and product-preserving
changes of basis lead to nuclear-norm noise bounds
\cite{qmm,nuclearcompanion}.  These representation choices determine the rows
and step sizes supplied to the fixed-grid problem studied here.

\section{Limitations and open problems}
\label{sec:disc-discussion}

The dynamic construction computes, for each observed row and block, a
rounding whose error is provably within $rR^2/4$ of the best admissible
choice.  Deployment must recover that accuracy within a matrix kernel's
reuse and time budgets.  The static identity already respects weight reuse,
but its empirical metric must transfer from calibration to future inputs.

\paragraph{Rounding cost and tile reuse.}
Once the rank is known, a direct null-space implementation computes a
direction on $r+1$ active coordinates at each step.  Recomputing its singular value decomposition costs
$O((|U|-r)(pr^2+r^3))$ arithmetic operations; completing the final choices
costs $O((|U|+r)p)$.  Computing the rank initially by SVD adds
$O(|U|p\min\{|U|,p\})$ work.  A clipping-aware start also pays for the box-constrained
quadratic solve.  The underlying row--block product costs $O(Kp)$.
Thus the polynomial-time guarantee concerns accuracy, not throughput.
Appendix~\ref{sec:disc-algorithms} relates measured product drift to the
exact-arithmetic certificate, and \Cref{sec:disc-exp-dynamic} reports the
implementation and timing protocol.

For $B$ output blocks per row, charging the whole rounding cost
$C_{\mathrm{round}}$ to each block gives $O(B C_{\mathrm{round}})$ work.
Materializing all activation copies takes $O(BK)$ codes per row; streaming
one copy at a time reduces live storage to $O(K)$ but retains repeated
rounding and tile movement.  An optimized implementation would need to
reuse block factorizations, update null directions incrementally, or find
one rounded tile that controls several block images.  The timing study
therefore establishes the computational cost that such a kernel must
amortize; kernel-level acceleration is outside the present study.

\paragraph{Calibration and evaluation scope.}
With a known input distribution, the static loss identity is exact.
An empirical second moment estimates that loss, and a change in the input
mean can change which rounding pattern minimizes it.  Held-out product error
therefore tests transfer beyond the calibration objective.  The Digits study
isolates this mean-error mechanism in a small fitted linear model using one
local-search start.  Its matched-bias comparisons and stopping checks
support that diagnostic; model-wide accuracy and speed require experiments
on frozen transformer projections and complete quantized networks.
Joint optimization of scales, clipping, or transforms would also change the
feasible set, as described in \Cref{sec:disc-static}.

\paragraph{Additive approximation: what is already determined.}
Write $\mathcal D(V)=\norm{V}_F^2/12$ for the unit-step dither scale.
The sufficient effective-row condition is not necessary for an approximation
scheme: \Cref{cor:disc-heavy-enumeration} gives error
$\operatorname{OPT}_{\mathrm{dyn}}+\varepsilon\mathcal D(V)+\eta$
in $2^{\lceil3r/\varepsilon\rceil}$ times polynomial work by enumerating
rows with large energy.  This holds for arbitrary fractional targets and
fixed offsets.

Unrestricted rank rules out such a guarantee in polynomial time for every
fixed $\varepsilon>0$, unless $\mathrm P=\mathrm{NP}$.  Indeed, the
incidence construction in \Cref{thm:disc-maxcut} has
$\mathcal D(H)=|E|/6$ and loss $|E|-\operatorname{cut}$.
An additive error of $\varepsilon\mathcal D(H)$ would give a cut at least
$(1-\varepsilon/3)$ times optimal, since an optimal cut contains at least
$|E|/2$ edges.  Taking $\varepsilon<3/17$ contradicts the $16/17$
approximation barrier for \textsc{Max-Cut} \cite{hastad2001optimal}.

At rank one with $\theta=\frac12\mathbf1$ and no fixed offset, all
coordinates must receive a
sign, giving the number-partitioning problem.  The small residuals analyzed
for Karmarkar--Karp concern random inputs \cite{karmarkarkarp1983}.
Worst-case number balancing with coefficients in $\{-1,0,1\}$ may leave
coordinates unused \cite{hoberg2017balancing}; this is a different feasible
set.  For example, an odd number of unit coefficients has full-sign residual
at least one, while two opposite signs and zeros elsewhere cancel exactly.
These results do not settle arbitrary-target rounding with all coordinates
assigned.  The remaining questions concern computational parameters, reuse,
and the output distribution of a rounding rule.

\begin{openproblem}[Precision-aware parameterization]
\label{op:disc-parameterization}
\Cref{thm:disc-partition} rules out even XP algorithms for the zero-error
decision problem under the rank parameter, and hence also rank-only
fixed-parameter tractability, unless $\mathrm{P}=\mathrm{NP}$.  Which
combinations of rank, coefficient magnitude, target precision, or a promised
nonzero gap admit pseudo-polynomial algorithms or approximation schemes?
\end{openproblem}

\begin{openproblem}[Reusable dynamic rounding]
\label{op:disc-reuse}
Which joint structure of the output blocks lets one rounded row improve on
the dither scale for every block while preserving activation-tile reuse?
Separate low ranks and large effective row counts are insufficient.
For example, let the $K$ blocks be the scalar-output columns of a
$K\times K$ Hadamard matrix $H$, with $H H^\top=K I$ and
$\theta=\frac12\mathbf1$.  Every block has rank one and
$K_{\mathrm{eff}}=K$, yet every shared sign vector $s$ satisfies
$\|H^\top s/2\|_2^2=K^2/4$, three times the sum of the block dither scales.
At least one block must exceed its reference.
A useful guarantee must therefore account for block count and joint image
geometry as well as the separate block ranks.  The directional bounds in
Appendix~\ref{sec:disc-extensions} provide one way to express that joint
control.
\end{openproblem}

\begin{openproblem}[Verifiable spectral tails]
\label{op:disc-tail}
Can a polynomial-time rounding rule guarantee both the head bound
\eqref{eq:disc-head} and a nontrivial tail error-second-moment bound of the
form \eqref{eq:disc-tail-assumption}?  The diagnostic in
Appendix~\ref{sec:disc-exp-gsw-tail} places every finite-draw Gram--Schmidt estimate
above both dither and the exact independent same-grid control.  Determining
a valid population tail scale for the same rounding walk would make the
spectral estimate applicable to that output law.
\end{openproblem}

\section{Conclusion}
\label{sec:disc-conclusion}

Fixed-grid rounding can be designed around the output in which scalar
errors combine.  For an observed row and block, directions that leave the
product unchanged resolve most choices before any completion error is
paid.  For reusable weights, the deployed bias convention determines whether
the calibration loss must retain the input mean.  The experiments expose
both mechanisms and compare the resulting errors with independent rounding
on the same grids.

The next practical step is to preserve these cancellations in a rounded
activation tile shared across output blocks, accounting jointly for rounding
time, tile movement, and product error.  The present results establish when
coordinated rounding can provably reduce product error and identify the reuse
constraints that such an implementation must preserve.

\section*{Acknowledgments}

This work was supported by the 2026 Laboratory Directed Research and
Development (LDRD) FORSEE initiative, ``CCSD Core: Foundational Research for
Smart Extreme-scale Ecosystems,'' at Oak Ridge National Laboratory.

Generative AI tools were used for LaTeX formatting, copy-editing, figure
generation, and assistance with scripts for numerical checks. They
were also used to help locate and organize related literature. All theorems,
proofs, citations, and claims of novelty were checked by the author, who
takes full responsibility for the content of this paper.

\printbibliography[title={References}]

\appendix
\section{Extensions toward reusable rounding}
\label{sec:disc-extensions}

A reusable rounded row must control several outputs at once.  A randomized
construction can supply directional concentration for that purpose, while a
separate spectral-tail estimate addresses the error omitted by a low-rank
metric.  The diagnostic below tests the tail assumption and finds empirical
operators above both independent-rounding references.

\subsection{A randomized Gram--Schmidt guarantee}

For reuse, the useful guarantee is that one rounding draw controls many
specified directions.  The Gram--Schmidt walk provides
\emph{sub-Gaussian tails}: deviation probabilities decay exponentially in
the squared error in each direction.  It rounds directly from the fractional
start, replacing the null-space reduction.

\begin{corollary}[Gram--Schmidt walk from an arbitrary fractional start]
\label{cor:disc-gs}
Let $V\in\R^{|U|\times p}$ have rank $r$ and rows $v_k^\top$ satisfying
$\norm{v_k}_2\le R$.  For every $\theta\in[0,1]^{|U|}$ and
$\delta\in(0,1)$, the randomized polynomial-time Gram--Schmidt walk returns
$\chi\in\{0,1\}^{|U|}$ such that, with probability at least $1-\delta$,
\begin{equation}
  \norm*{\sum_{k\in U}(\chi_k-\theta_k)v_k}_2
  \le
  2\sqrt{20}\,R
  \sqrt{r\log 5+\log(2/\delta)} .
  \label{eq:disc-gs-high-probability}
\end{equation}
It also satisfies
\begin{equation}
  \E\norm*{\sum_{k\in U}(\chi_k-\theta_k)v_k}_2^2
  \le 10R^2r .
  \label{eq:disc-gs-second-moment}
\end{equation}
These bounds control the active contribution.  An affine fixed term
$g^{\mathrm{fix}}$ can instead be included by starting the Gram--Schmidt walk
at the continuous-relaxation optimum, as in \Cref{thm:disc-additive}; the
corresponding expected additive term is $10R^2r$.
\end{corollary}

The proof appears in Appendix~\ref{app:disc-gs-proof}.

The expected Euclidean certificate is $40$ times the deterministic
certificate.  The additional information is directional: for any $B$
prescribed unit vectors $u_1,\ldots,u_B$, a union bound gives
\[
  \max_{b\le B}|\langle u_b,V^\top(\chi-\theta)\rangle|
  \le R\sqrt{20\log(2B/\delta)}
\]
with probability at least $1-\delta$.  Thus one draw can certify several
specified projections with logarithmic dependence on their count.
Turning this control into a shared rounding for output blocks is the reuse
question in \Cref{op:disc-reuse}.

\subsection{A conditional spectral head--tail estimate}

A low-rank metric reduces the rounding dimension by retaining its leading
eigenspace.  The full loss also includes the discarded directions.  Bounding
the rounding error's second moment in those directions accounts for this
remaining contribution.

\begin{remark}[Conditional spectral head--tail estimate]
\label{rem:disc-head-tail}
Let every coordinate be active, with $e^{\mathrm{fix}}=0$ and
$D=\Delta I$.  Write the rank-$r$ truncation of the static metric as
$M_r=\Phi_r\Phi_r^\top$ and the remainder as $M_{>r}=M-M_r$.  Let
$\beta_{\max,r}=\max_k\norm{(\Phi_r)_{k,:}}_2$ be the largest row norm of the
head factor.  Applying the Gram--Schmidt walk to the rows of $\Delta\Phi_r$
gives
\begin{equation}
  \E[e^\top M_re]
  \le 10\Delta^2r\beta_{\max,r}^2 .
  \label{eq:disc-head}
\end{equation}

To control the tail, let $\Pi_{>r}$ project onto the remaining eigenspace and
let $\Sigma_e=\E[ee^\top]$ be the error second moment.  If
\begin{equation}
  \Pi_{>r}\Sigma_e\Pi_{>r}
  \preceq \sigma_{\mathrm{tail}}^2\Pi_{>r}
  \label{eq:disc-tail-assumption}
\end{equation}
holds, then combining the head bound with the tail hypothesis gives
\begin{equation}
  \E[e^\top Me]
  \le
  10\Delta^2r\beta_{\max,r}^2
  +
  \sigma_{\mathrm{tail}}^2\operatorname{tr}(M_{>r}).
  \label{eq:disc-head-tail}
\end{equation}

The proof of \eqref{eq:disc-head-tail} appears in
Appendix~\ref{app:disc-head-tail-proof}.

Condition \eqref{eq:disc-tail-assumption} bounds error energy in every tail
direction.  It must hold for the same Gram--Schmidt output used in
\eqref{eq:disc-head}.  Independent zero-mean errors with per-coordinate second moment
$\sigma_{\mathrm{tail}}^2$ satisfy this isotropy condition.  Independent
dither has scale $\Delta^2/12$; applying that scale to the Gram--Schmidt
output requires a separate bound for its own projected second moment.
The diagnostic in Appendix~\ref{sec:disc-exp-gsw-tail} estimates that operator from
finite draws, while \eqref{eq:disc-tail-assumption} concerns the population
second moment.
\end{remark}

\subsection{Spectral-tail diagnostic: estimates exceed the dither scale}
\label{sec:disc-exp-gsw-tail}

The exact-rank blocks of \Cref{sec:disc-exp-dynamic} have zero spectral
tail.  This diagnostic instead uses full-rank metrics and tests whether the
same output that satisfies the head guarantee also has tail operator
$\lambda_{\max}(\Pi_{>r}\Sigma_e\Pi_{>r})$ at the dither scale
$\Delta^2/12$.

\paragraph{Design.}
We set $K=\ifgswtailresults\DiscrepancyGSWTailK\else32\fi$ and use head ranks
$r\in\{2,4,8,16\}$.  The spectra are flat, polynomial
($\lambda_j=1/j$), and geometric ($\lambda_j=0.75^{j-1}$).  For each
eigenbasis and rank, we draw one fixed fractional target with coordinates uniform
on $[0,1]$ and reuse that paired target across the three spectra.  From each
fixed metric and fractional target, we draw
\ifgswtailresults\DiscrepancyGSWTailDraws\else1024\fi{} independent outputs of
Algorithm~1 in \cite{gramschmidt}, form the empirical error second moment, and
report its tail operator divided by $\Delta^2/12$.  For the same eigenbasis and
target, we compute the analytic operator for independent same-grid Bernoulli
rounding and a draw-matched sampled control.  Estimates from the first 256,
512, and 1,024 outputs track convergence with
draw count.  The full protocol and numerical
validation appear in
Appendix~\ref{app:disc-gsw-tail-protocol}.

\paragraph{Results.}
\ifgswtailresults
The empirical tail operators exceed both the dither scale and the exact
same-grid Bernoulli control in every tested configuration
(\Cref{fig:disc-gsw-tail}).  The diagnostic therefore identifies the tail
operator, alongside the head error, as a concrete quantity that any
reusable-rounding theory must control (\Cref{op:disc-tail}): neither
reference scale can serve as $\sigma_{\mathrm{tail}}^2$ in
\eqref{eq:disc-head-tail}, and the head guarantee \eqref{eq:disc-head}
covers only the retained directions.  Across the
four-instance medians, the operator ratio ranges from
\DiscrepancyGSWTailMinimumMedianRatio{} to
\DiscrepancyGSWTailMaximumMedianRatio{}.  The analytic same-grid independent
control already ranges from
\DiscrepancyGSWTailMinimumAnalyticBernoulliRatio{} to
\DiscrepancyGSWTailMaximumAnalyticBernoulliRatio{} times the dither variance,
so part of the dither gap comes from the adjacent-level marginals rather than
Gram--Schmidt dependence.  After division by that control, the four-instance
median Gram--Schmidt ratios range from
\DiscrepancyGSWTailMinimumMedianOverAnalyticBernoulli{} to
\DiscrepancyGSWTailMaximumMedianOverAnalyticBernoulli{}.  Replacing the exact
control by its draw-matched empirical estimate gives median ratios from
\DiscrepancyGSWTailMinimumMedianOverSampledBernoulli{} to
\DiscrepancyGSWTailMaximumMedianOverSampledBernoulli{}, providing a
draw-count-matched check on finite-sample operator inflation.  The relative
change from 512 to 1,024 Gram--Schmidt draws is
\DiscrepancyGSWTailMedianHalfToFullChangePercent\% at the median and
\DiscrepancyGSWTailMaximumHalfToFullChangePercent\% at maximum across
configurations.
\fi

\begin{figure}[H]
  \centering
  \IfFileExists{figures/discrepancy/gsw_tail_covariance.pdf}{%
    \includegraphics[width=0.52\linewidth]
      {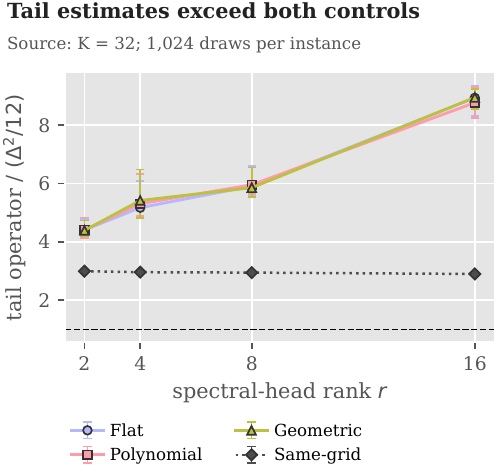}%
  }{}
  \caption{Tail operators divided by the dither variance.  Colored markers are
  empirical Gram--Schmidt medians for flat ($\lambda_j=1$), polynomial
  ($\lambda_j=1/j$), and geometric ($\lambda_j=0.75^{j-1}$) spectra.
  The gray diamonds give the analytic independent
  same-grid Bernoulli control for the same eigenbases and fractional targets.
  Bars span interquartile ranges across four paired instances, and each
  Gram--Schmidt estimate uses 1,024 outputs.  The horizontal dashed line is the
  dither scale $\Delta^2/12$.}
  \label{fig:disc-gsw-tail}
\end{figure}

\section{Algorithmic details and deferred proofs}
\label{sec:disc-algorithms}

Coordinates are ordered by their original indices, and ties are resolved in
favor of the
smaller index or lower endpoint.

\paragraph{Algorithm 1: null-space reduction.}
Given $V$ and $x^{\mathrm{init}}\in[0,1]^{|U|}$, set
$x\gets x^{\mathrm{init}}$ and repeat:
\begin{enumerate}
\item Let $J=\{k:0<x_k<1\}$.  If $|J|\le\operatorname{rank}(V)$, return $x$.
\item Let $T$ be the $r+1$ smallest-index coordinates in $J$, where
  $r=\operatorname{rank}(V)$.  Row-reduce $V_T^\top$ using the coordinate
  order on $T$.  Select the
  smallest-index free variable, set it to one and the other free variables to
  zero, solve for the pivot variables, and extend the resulting nonzero null
  vector $h$ with zeros outside $T$.  If needed, change its sign so the first
  nonzero entry is positive.
\item Compute the largest $\alpha>0$ for which $x+\alpha h\in[0,1]^{|U|}$:
  \[
    \alpha
    =
    \min\left\{
      \frac{1-x_k}{h_k}:h_k>0,\quad
      \frac{-x_k}{h_k}:h_k<0
    \right\}.
  \]
  Set $x\gets x+\alpha h$.  At least one additional coordinate becomes
  integral, while $V^\top x$ remains unchanged.
\end{enumerate}
For floating-point experiments, row reduction is replaced by an SVD of the
active submatrix, with the tolerance recorded in the artifact.

\paragraph{Finite-precision product drift.}
The reduction's numerical error can be checked in output space.  Let
$\widetilde x$ be its start and $\bar x\in[0,1]^{|U|}$ its computed
endpoint, and measure
$\delta=\|V^\top(\bar x-\widetilde x)\|_2$.  The triangle inequality gives
\[
  \sqrt{\mathcal L(\bar x)}
  \le\sqrt{\mathcal L(\widetilde x)}+\delta.
\]
Applying the exact completion rule to $\bar x$, with fractional set $S$,
therefore yields the a posteriori bound
\begin{equation}
  \mathcal L(\chi)
  \le\mathcal L(\widetilde x)
     +2\delta\sqrt{\mathcal L(\widetilde x)}+\delta^2
     +\frac14\sum_{k\in S}\|v_k\|_2^2.
  \label{eq:disc-drift}
\end{equation}
If $|S|\le r$, the last term is at most $rR^2/4$.
For computed steps $\alpha_t h_t$, local residuals satisfy
$\delta\le\sum_t|\alpha_t|\|V^\top h_t\|_2$ plus the product-image effect
of arithmetic and endpoint snapping errors.  Measuring $\delta$ directly
includes all these effects.  This is a check on the computed path; a forward
stability guarantee would additionally need to control the residuals and step
lengths.  The experimental validation also recomputes the final objective
to check finite-precision completion decisions.  These measured norms are
floating-point diagnostics; a rigorous numerical certificate would also
bound the rounding error in evaluating them.

\paragraph{Algorithm 2: conditional-expectation completion.}
Given the endpoint $x^\star$ from Algorithm~1, let $S$ be its fractional set.
Copy its integral coordinates to $\chi$, and initialize
$y\gets g^{\mathrm{fix}}+V^\top(x^\star-\theta)$.  Process $S$ in coordinate
order.  At coordinate $k$, form
\[
  y_0=y-x_k^\star v_k,
  \qquad
  y_1=y+(1-x_k^\star)v_k .
\]
The still-random coordinates contribute the same remaining variance to both
conditional expectations, so choose $\chi_k=0$ when
$\norm{y_0}_2^2\le\norm{y_1}_2^2$ and $\chi_k=1$ otherwise.
Replace $y$ by the selected vector.  The final $\chi$ satisfies the guarantee in
\Cref{lem:disc-ce}.

\paragraph{Algorithm 3: affine-aware additive approximation.}
Find a feasible $\widetilde x\in[0,1]^{|U|}$ satisfying
\[
  \mathcal L(\widetilde x)
  \le \min_{x\in[0,1]^{|U|}}\mathcal L(x)+\eta,
  \qquad
  \mathcal L(x)=\norm{g^{\mathrm{fix}}+V^\top(x-\theta)}_2^2.
\]
Apply Algorithm~1 starting from $\widetilde x$, then apply Algorithm~2 to the
resulting fractional endpoint.  This procedure achieves the guarantee in
\Cref{thm:disc-additive}, with an additional $+\eta$ when the continuous
relaxation is solved approximately.

\paragraph{Algorithm 4: one-flip metric descent.}
For each static-factor column, initialize at round-to-nearest.  With current
error $e$, metric $M$, and gradient $q=Me$, evaluate every active coordinate's
exact flip increment
\[
  \delta_i=2s_iq_i+s_i^2M_{ii},
\]
where $s_i$ is the change from the current adjacent level to the other one.
Flip the coordinate with the most negative $\delta_i$, breaking ties by
smallest index, update $e$ and $q$ exactly, and rescan.  Stop when all
$\delta_i$ are nonnegative up to the recorded floating-point tolerance.  Each
accepted step strictly decreases the objective over a finite feasible set, so
the procedure terminates at a one-flip local optimum.

\subsection{Proof of the Gram--Schmidt corollary}
\label{app:disc-gs-proof}

\begin{proof}[Proof of \Cref{cor:disc-gs}]
If $R=0$, both conclusions are immediate.  Otherwise apply the
Gram--Schmidt walk in the $r$-dimensional row span.  Theorem~1.4 of
\cite{gramschmidt} states that, for vectors of norm at most one and an
arbitrary start $x^0\in[-1,1]^{|U|}$, the walk returns
$x\in\{-1,1\}^{|U|}$ for which
$\sum_k(x_k-x_k^0)v_k$ is $\sqrt{40}$-sub-Gaussian.  Here a random vector
$Z$ is $\sigma$-sub-Gaussian when
$\E\exp\langle u,Z\rangle\le
\exp(\sigma^2\norm{u}_2^2/2)$ for every $u$.
Apply that theorem to $v_k/R$ from
$x^0=2\theta-\mathbf1$, and set $\chi=(x+\mathbf1)/2$.  The factor $1/2$
makes
\[
  Z=\sum_k(\chi_k-\theta_k)v_k
\]
$\sqrt{10}R$-sub-Gaussian.  A $1/2$-net of the unit sphere in the
row span\footnote{A $1/2$-net is a finite set on the sphere such that every
point of the sphere is within Euclidean distance $1/2$ of a point in the set.}
has at most $5^r$ points and satisfies
$\norm{Z}_2\le2\max_u|\langle u,Z\rangle|$.  The directional tail bound and a
union bound give \eqref{eq:disc-gs-high-probability}.  Summing directional
second moments over an orthonormal basis of the row span gives
\eqref{eq:disc-gs-second-moment}.  These bounds are an algorithmic
vector-balancing consequence of the Gaussian-measure method \cite{banaszczyk}.

For a start $y\in[0,1]^{|U|}$, write
$g_y=g^{\mathrm{fix}}+V^\top(y-\theta)$.  By the same theorem,
$Z=V^\top(\chi-y)$ has mean zero and satisfies
$\E\norm{Z}_2^2\le10R^2r$.  Hence
$\E\norm{g_y+Z}_2^2\le \mathcal L(y)+10R^2r$.
Choosing the continuous-relaxation optimum for $y$ proves the affine statement.
\end{proof}

\subsection{Proof of the conditional head--tail bound}
\label{app:disc-head-tail-proof}

\begin{proof}[Proof of \Cref{rem:disc-head-tail}]
Equation \eqref{eq:disc-head} follows by applying
\eqref{eq:disc-gs-second-moment} in the head eigenspace.  Orthogonality
of the head and tail decomposes the expected quadratic form.  For the tail,
\[
  \E[e^\top M_{>r}e]
  =\operatorname{tr}(M_{>r}\Sigma_e)
  =\operatorname{tr}\!\left(
    M_{>r}\Pi_{>r}\Sigma_e\Pi_{>r}
  \right)
  \le \sigma_{\mathrm{tail}}^2\operatorname{tr}(M_{>r}).
\]
Combining this tail bound with \eqref{eq:disc-head} proves
\eqref{eq:disc-head-tail}.
\end{proof}

\section{Gram--Schmidt spectral-tail protocol}
\label{app:disc-gsw-tail-protocol}

This diagnostic implements Algorithm~1 of \cite{gramschmidt}.  At each step,
the largest alive index is the pivot.  A least-squares solve projects its
vector orthogonally to the span of the other alive vectors, and the two
boundary step lengths are sampled with the published mean-zero
probabilities.

\paragraph{Numerical validation.}
We independently reconstruct one update by a singular value decomposition,
verify that each draw terminates at a sign vector inside the cube, and check
the prescribed marginals from arbitrary fractional starts.  We also record
an orthogonality residual at every update; its maximum is
\ifgswtailresults\(\DiscrepancyGSWTailMaximumOrthogonalityError\)\else
\(1.3\times10^{-11}\)\fi.

\paragraph{Instances and measurement.}
For each of four seeded Haar eigenbases in dimension 32, we draw one fractional
target with independent uniform coordinates for every head rank
$r\in\{2,4,8,16\}$.  The target vector remains fixed across repeated walk outputs
and is shared across the three spectral profiles at the same eigenbasis and
rank.
The three eigenvalue sequences are $\lambda_j=1$, $\lambda_j=1/j$, and
$\lambda_j=0.75^{j-1}$.  With $U_{>r}$ denoting the tail eigenvectors, the
experiment uses the empirical second moment
\[
  \widehat\Sigma_{>r}
  =\frac1T\sum_{t=1}^T
    (U_{>r}^{\top}e_t)(U_{>r}^{\top}e_t)^{\top},
  \qquad T=1024,
\]
and reports $\lambda_{\max}(\widehat\Sigma_{>r})/(\Delta^2/12)$.  This is an
empirical second moment over the walk's internal randomness for a fixed metric
and fixed fractional target, matching the distribution required by
\eqref{eq:disc-tail-assumption}.

Two paired controls separate the choice of adjacent-level marginals from the
walk's dependence.  Independent same-grid rounding has the exact second moment
\[
  \Sigma_{\mathrm{Bern}}
  =\Delta^2\operatorname{diag}\!\left(\theta_k(1-\theta_k)\right),
\]
so we compute
$\lambda_{\max}(U_{>r}^{\top}\Sigma_{\mathrm{Bern}}U_{>r})$ without Monte
Carlo.  We also draw $T=1024$ independent Bernoulli roundings for the same
eigenbasis and target, giving a control with the same draw count and projected
dimension as the Gram--Schmidt estimate.  This paired control checks the scale
of finite-sample operator inflation.

\paragraph{Variation and draw-count convergence.}
Prefix estimates at $T\in\{256,512,1024\}$ record convergence with draw count.
The interquartile ranges summarize variation across the four observed
instances.  Prefix comparisons measure the sensitivity of the empirical
operator to draw count; \Cref{rem:disc-head-tail} states the population
condition to which these estimates relate.

\section{Dynamic implementation and numerical checks}
\label{app:disc-dynamic-protocol}

\paragraph{Implementation and work.}
The reference code uses Python and NumPy float64 arrays on an Apple M3 Max
CPU.  It
computes numerical rank once per target, at cost
$O(|U|p\min\{|U|,p\})$, then takes the first $r+1$
fractional coordinates in coordinate order and calls
\texttt{numpy.linalg.svd} on their $p\times(r+1)$ transposed row matrix
at each step.  For $p\ge r+1$, economy SVD supplies the right null vector
without allocating a $p\times p$ left factor; for $p=r$, full right singular
vectors include the null direction.  Each step costs $O(pr^2+r^3)$ arithmetic operations, and the $K=1024,r=16$
case takes 1008 steps.  A call also scans the current fractional coordinates
and snaps entries within $2\times10^{-11}$ of an endpoint.  These scans and
Python dispatch contribute to elapsed time.

The reported time charges this work to every fractional target.  With $B$
output blocks, the reference code repeats it $B$ times per activation row.
Its raw CSV separates reduction and completion times, and the JSON records
same-process float64 row-block and float32 matrix-product timings.
The run requests one thread through the BLAS/OpenMP environment variables;
the CPU, package, BLAS, and runtime-threadpool metadata
are saved in \texttt{discrepancy\_summary.json} and the manifest.

\paragraph{Residual accounting.}
The code records the final product drift and the largest relative null
residual over all reduction steps.  It also checks the final computed
objective against the conditional-expectation bound formed from the stored
endpoint.  Equation~\eqref{eq:disc-drift} explains how drift contributes to
the a posteriori certificate; the checks use the recorded floating-point
tolerances for its numerical evaluation.

\paragraph{Scaling and conditioning diagnostic.}
The diagnostic protocol combines condition numbers $1,10^6,10^{12}$ with
overall scales $10^{-8},1,10^8$ in nine $64\times32$ rank-four instances.
The orthonormal factors and fractional target are shared.  It records the
fractional count, relative step null residual, and relative product drift
$\|V^\top(x_T-x_0)\|_2/(\|V\|_F\|x_T-x_0\|_2)$.
These quantities measure arithmetic drift for specified ranks;
numerical-rank selection near its singular-value threshold remains a
separate modeling decision.

\section{Static-study protocol}
\label{app:disc-static-protocol}

\paragraph{Data and fitted model.}
We use scikit-learn Digits ($8\times8$ grayscale inputs)
\cite{alpaydin1998digits,pedregosa2011scikit}, scaling pixels by $1/16$ without
centering.  Five deterministic stratified splits allocate 60\% to fitting,
20\% to calibration, and 20\% to held-out evaluation.  Each split fits a
ridge least-squares ten-class linear model with penalty $10^{-3}$ and no intercept.
\ifdiscresults
The fit, calibration-pool, and held-out sizes are
\DiscrepancyStaticFitSize{}, \DiscrepancyStaticCalibrationPoolSize{}, and
\DiscrepancyStaticTestSize{}.
\fi

\paragraph{Grids and calibration subsets.}
Each output column receives a signed symmetric 3- or 4-bit grid whose scale is
the column's maximum absolute weight divided by $2^{b-1}-1$, where $b$ is
the bit width.  Across fits,
this choice
leaves 580--600 active and 40--60 exactly representable weights, with no clipped
entries among 640 total.  From each calibration pool, we draw twenty
class-stratified subsets of 128 examples without replacement within a subset;
different subsets may overlap.  We also use the full pool of
\ifdiscresults\DiscrepancyStaticFullCalibrationSize\else359\fi{} examples.

\paragraph{Solver and evaluation.}
Best-improvement one-flip descent starts from round-to-nearest.  It scans
every active coordinate, breaks ties by the smallest coordinate index, and stops
when no flip decreases the quadratic objective beyond the recorded
floating-point tolerance.  We compare uncentered $M=\E[x^\top x]$ with
centered covariance under fixed bias and optimal bias correction.  Both centered methods use the same weights; the corrected version adds
$c_j=-\mu e_j$ to output column $j$, where $\mu$ is the calibration mean
and $e_j$ its weight-error column.  The RTN-plus-bias control applies the
identical formula to RTN's own error columns, with the same calibration
inputs.  The held-out mean is never used to choose a correction.  Held-out product MSE is divided by split-matched round-to-nearest error
at the same bit width.  The artifact records flip counts, stopping
certificates, end-to-end times, and the floating-point tolerance.

\paragraph{An absolute mean-term scale.}
For split zero, the full calibration pool, four-bit centered rounding, and
the class-zero error column $e_0$, the measured ratio is
\[
 \frac{(\mu e_0)^2}{e_0^\top(M-\mu^\top\mu)e_0}=2.524.
\]
Averaged over all ten outputs, the mean term contributes $0.08949$ MSE and
the covariance term contributes $0.02664$ MSE.  Centering drops the larger
contribution from the calibration objective; the fixed-bias evaluation
retains both.  The raw CSV records the decomposition for every split and
subset, and the protocol identifies this fitted model and its rounded
endpoint deterministically.

\paragraph{Paired summaries.}
For each split, we pair methods on the same calibration subset, take the
median of the subset-level held-out MSE ratios, then summarize these five
split-level ratios.  The mean column is the arithmetic mean of those five
ratios; wins count splits whose paired median difference favors the first
method.  The CSV contains every subset and the JSON lists every split,
absolute MSE difference, and two-sided sign-test $p$-value.  With five
non-tied split outcomes, even five wins give $p=0.0625$; the overlapping
training and evaluation splits further limit an independence interpretation.
The intervals therefore summarize the observed splits descriptively.

\begin{table}[tbp]
  \centering
  \scriptsize
  \begin{tabular}{rrlrrr}
\toprule
Bits & Cal. & Paired comparison & Ratio [95\% RI] & Mean & Wins \\
\midrule
$3$ & 128 & Centered / RTN & $1.264\,[0.933,2.035]$ & $1.325$ & $1/5$ \\
$3$ & 128 & Centered+bias / RTN+bias & $0.869\,[0.721,0.973]$ & $0.863$ & $5/5$ \\
$3$ & full & Centered / RTN & $1.250\,[0.884,2.027]$ & $1.298$ & $1/5$ \\
$3$ & full & Centered+bias / RTN+bias & $0.842\,[0.694,0.946]$ & $0.827$ & $5/5$ \\
$4$ & 128 & Centered / RTN & $1.148\,[0.867,1.753]$ & $1.314$ & $1/5$ \\
$4$ & 128 & Centered+bias / RTN+bias & $0.667\,[0.615,0.687]$ & $0.656$ & $5/5$ \\
$4$ & full & Centered / RTN & $1.260\,[0.832,1.835]$ & $1.363$ & $1/5$ \\
$4$ & full & Centered+bias / RTN+bias & $0.627\,[0.571,0.650]$ & $0.619$ & $5/5$ \\
\bottomrule
\end{tabular}

  \caption{Paired held-out product MSE comparisons on Digits.  Ratios compare
  the methods shown in each row, so corrected weights are compared with
  corrected RTN.  RI denotes a bootstrap resampling interval over five
  split-level median ratios.}
  \label{tab:disc-static-paired}
\end{table}

\end{document}